\documentclass[twoside,twocolumn]{article}

\usepackage[margin=1in]{geometry}

\usepackage{amsmath, amssymb, amsfonts}
\usepackage{algorithm}
\usepackage{graphicx}
\usepackage[english]{babel}
\usepackage{booktabs}
\usepackage{bbm}
\usepackage{caption}
\newtheorem{theorem}{Theorem}
\newtheorem{lemma}{Lemma}
\newtheorem{assumption}{Assumption}
\newtheorem{proposition}{Proposition}
\newtheorem{proof}{Proof}
\newtheorem{example}{Example}
\newtheorem{remark}{Remark}

\usepackage{xcolor}

\usepackage[round,authoryear]{natbib}  
\usepackage{url}
\usepackage{float}
\usepackage{hyperref}
\usepackage{cleveref}
\usepackage{booktabs, multirow}
\usepackage{placeins}
\usepackage{longtable}
\usepackage{pdflscape}
\usepackage{algorithm}
\usepackage[noend]{algpseudocode}
\usepackage{subcaption}
\usepackage{cuted}
\usepackage{titletoc}   
\usepackage{tcolorbox}  
\newtcolorbox{apptocbox}{
  colback=gray!5, colframe=black!20, boxrule=0.4pt,
  left=6pt, right=6pt, top=6pt, bottom=6pt,
  title=Appendix Contents
}

\allowdisplaybreaks
\algrenewcommand\alglinenumber[1]{\scriptsize #1}

\begin{document}
\captionsetup[algorithm]{font=small,skip=4pt}


\title{Ensuring Calibration Robustness in Split Conformal Prediction Under Adversarial Attacks}

\author{%
  \begin{tabular}{cc}
    Xunlei Qian & Yue Xing \\
    \multicolumn{2}{c}{Michigan State University}
  \end{tabular}
}
\date{}  

\maketitle


\begin{abstract}

 Conformal prediction (CP) provides distribution-free, finite-sample coverage guarantees but critically relies on exchangeability, a condition often violated under distribution shift. We study the robustness of split conformal prediction under adversarial perturbations at test time, focusing on both coverage validity and the resulting prediction set size. Our theoretical analysis characterizes how the strength of adversarial perturbations during calibration affects coverage guarantees under adversarial test conditions. We further examine the impact of adversarial training at the model-training stage. Extensive experiments support our theory: (i) Prediction coverage varies monotonically with the calibration-time attack strength, enabling the use of nonzero calibration-time attack to predictably control coverage under adversarial tests; (ii) target coverage can hold over a range of test-time attacks: with a suitable calibration attack, coverage stays within any chosen tolerance band across a contiguous set of perturbation levels; and (iii) adversarial training at the training stage produces tighter prediction sets that retain high informativeness.
\end{abstract}

\section{Introduction}
With the rapid adoption of deep learning in high-stakes domains, such as survival analysis \citep{candes2023conformalized} and chest X-ray report generation \citep{gui2024conformal}, uncertainty quantification has become an increasingly intricate and essential task for ensuring 
reliable predictive performance. Conformal prediction (CP) \citep{vovk2005algorithmic} offers a powerful, distribution-free framework that provides finite-sample validity by constructing prediction sets calibrated to a user-specified coverage level.  For example, CP can support clinical decision workflows in which physicians must deploy complex models for diagnosis \citep{banerji2023clinical, kiyani2025decision, hullman2025conformal}.

Despite its flexibility, CP typically relies on independent and identically distributed (i.i.d.) or, more generally, at least exchangeable data. These assumptions are often violated in practice \citep{barber2023conformal}. When exchangeability is broken, the conformity-score distribution in the calibration set need not match that of the test data due to distribution shift commonly observed in real-world settings \citep{moreno2012unifying,sugiyama2007covariate,koh2021wilds}
A large body of work proposes mitigation strategies within the classical CP paradigm: some methods address random perturbations at test time \citep{gendler2022adversarially,yan2024provablyrobustconformalprediction}, while others demonstrate that adversarial perturbations can be particularly harmful for deep models \citep{kos2017delving}. These observations motivate the study of CP under non-exchangeable conditions.

On the other hand, Adversarial attack is a prominent form of distribution shift in which inputs are deliberately perturbed in worst-case directions to degrade performance \citep{biggio2013evasion,goodfellow2014explaining}. Although such perturbations are often imperceptible to humans \citep{wang2019improving}, the deep neural network or even transformer-based models like ViT(vanilla vision transformer) can suffer substantial drops in performance compared with unperturbed inputs \citep{shao2021adversarial,aldahdooh2021reveal,narodytska2017simple,miller2020adversarial}. Prior work spans attack generation \citep{xiao2018generating,assion2019attack}, adversarial training \citep{shafahi2019adversarial,andriushchenko2020understanding, shafahi2020universal}, and adversarial robustness \cite{bai2021recent, silva2020opportunities}.

In this paper, we examine test-time adversarial perturbations through the lens of conformal prediction. We analyze how such attacks affect validity (coverage) and efficiency (prediction set size), how calibration under adversarial perturbations can improve robustness at the test time, and how adversarial training will improve the efficiency of split conformal prediction. Our contributions are as follows:  

\begin{itemize}
    \item We provide a theoretical analysis of how the strength of adversarial perturbations applied to the calibration set affects prediction validity on both clean and adversarially perturbed test data. In particular, we show that test-time prediction accuracy is a monotone (non-decreasing) function of the calibration attack strength (Theorem~\ref{thm:theorem1}).
    \item We show that introducing a proper adversarial perturbation during calibration yields more robust predictions under potential test-time attacks. Specifically, for a 90\% target accuracy, CP attains a “reasonable” accuracy (e.g., 87\%–93\%) across a wider range of test-time attack strengths (Theorem~\ref{thm:theorem2}).
    \item We show that adversarial training can reduce prediction set size for Split conformal prediction, thereby improving both robustness and informativeness for deep neural networks (Theorem~\ref{thm:theorem3}).
    \item Experiments are conducted to verify the observations from the above theorems.
\end{itemize}

\section{Related Work}

\subsection{Adversarial training}

Adversarial examples have become central to the study of machine learning safety because of their implications for model robustness \citep{goodfellow2014explaining}. A growing body of research has explored the use of adversarial training as a primary defense. Despite substantial progress, modern models remain vulnerable to small, often imperceptible, perturbations in both regression and classification tasks. Recent theoretical and empirical results clarify a fundamental robustness–accuracy tradeoff in adversarially trained models. In linear settings, where analysis is more tractable, adversarial training improves robustness at the expense of standard accuracy. \citep{javanmard2020precise} precisely characterize this tradeoff and establish limits on the accuracy achievable by any algorithm. More recently, Xing et al. \citep{xing2024adversarial} showed that two-stage training with additional unlabeled data and pseudolabeling can enhance robustness even in two-layer neural networks. Taken together, these findings highlight the particular effectiveness of adversarial training for simple models, while also offering insights for extending such defenses to richer architectures. Other related literature can be found in \citep{goodfellow2015explaining,madry2018towards,zhang2019trades,shafahi2019free,carmon2019unlabeled,xie2020advprop}.

\subsection{Adversarially robust conformal prediction}
One critical high-stakes task in machine learning is constructing reliable prediction sets for classification tasks in the presence of adversarially perturbed test data. Recent work proposes several approaches toward this goal. \citep{gendler2022adversarially} demonstrated that standard CP fails to provide valid coverage under adversarial perturbations bounded in $\ell_2$ norm and introduce Randomized Smoothed Conformal Prediction (RSCP), which combines random smoothing with a modified thresholding rule to improve robustness. Extending this line, \citep{yan2024provablyrobustconformalprediction} proposed two approaches, Post-Training Transformation (PTT) and Robust Conformal Training (RCT), that significantly reduce the size of prediction sets while maintaining valid coverage under adversarial conditions. More recently, \citep{scholten2025provably} aggregate predictions from multiple models trained on partitioned datasets, yielding provable robustness to both test-time adversarial attacks and training-time data poisoning. Complementary developments include probabilistic robustness via quantile-of-quantiles, adversarially valid guarantees for sequential settings, certifiable reasoning with probabilistic circuits, Lipschitz-bounded models for scalable certification, and conformal training/attack frameworks that directly optimize set size under threat \citep{ghosh2023prcp,bastani2022mvp,areces2025onlinecp,kang2024colep,massena2025lipschitzrcp,bao2025opsa}. Collectively, these contributions advance provably robust conformal prediction in adversarial scenarios by addressing both validity and efficiency.

\section{Preliminaries}

The following section provides the general framework of split conformal prediction (Split CP) which we will adopt in our theorem and experiments.

\subsection{Split Conformal Prediction}\label{sec:split-cp}
Split CP provides distribution-free, finite-sample coverage guarantees at user-specified levels for both classification and regression \citep{Lei03072018,romano2020classification}. We adopt split CP as our primary procedure for constructing prediction sets. 

Consider a multi-class classification problem, where we aim to train a model \( f: X \to T \), with \( T = \{1, 2, \dots, K\} \) representing the set of class labels and \( X \) denoting the input space. To construct prediction sets with coverage guarantees, we adopt the \emph{Split conformal prediction} framework \citep{romano2020classification,Lei03072018}. The data is partitioned into three disjoint subsets: a training set \( I_1 \) with sample size $n_1$, a calibration set \( I_2 \) with sample size $n_2$, and a test set \( I_3 \) with sample size $n_3$.

We define the nonconformity score as
\[
S(x, y) = 1 - f_y(x),
\]
where \( f_y(x) = \mathbb{P}( y \mid x) \) is the estimated probability assigned by the trained model to class \( y \) given \( x \).

Given a new testing input \( x_{test} \) from testing set $I_{3}$ , the prediction set is constructed as
\begin{align*}
C(x_{test}) &= \left\{ y \in \{1, 2, \dots, K\} : f_y(x_{test}) \geq 1-Q \right\}\\ &\text{or}= \left\{ y \in \{1, 2, \dots, K\} : 1-f_y(x_{test}) \leq Q \right\},
\end{align*}
where \( Q \) is the \( (1 - \alpha)\left(1 + 1/|I_2| \right) \)-quantile of the nonconformity scores \( \{ S(x_i, y_i) : i \in I_2 \} \).

\begin{remark}
   There are several choices for the nonconformity score, including the HPS score proposed in \citep{Lei03072018} and the APS score introduced in \citep{romano2020classification}:
\begin{align*}
S_{\mathrm{HPS}}(x, y) &= 1 - f_y(x),\\
S_{\mathrm{APS}}(x, y) &= \sum_{y' \in [K]} \hat{\pi}_{y'}(x) \cdot \mathbf{1}_{\{f_{y'}(x) > f_{y}(x)\}} + f_{y}(x) \cdot u,
\end{align*}
where \( u \sim \mathcal{U}(0,1) \) is drawn from the uniform distribution on the unit interval. We focus on the HPS nonconformity score because it allows a more direct analysis of how adversarial attacks affect prediction accuracy.
\end{remark}

To evaluate Split CP, we report validity and efficiency. Validity refers to the proportion of test instances whose true label is contained in the conformal prediction set and should be near the nominal level $1-\alpha$; we summarize it via empirical coverage on the test set. Efficiency is quantified by the average prediction-set size.

Since Split CP is a post-training method, we need to train a model $f$ using adversarially attacked testing data. On the other hand, to ensure the adversarial robustness of the model, adversarial training is a commonly used approach. Before presenting the main theoretical results, we further introduce adversarial training below: 

\subsection{Adversarial training}
Adversarial training is formulated as follows \citep{xing2021generalization}: let \( \ell \) denote the loss function and let \( f_{\theta} : \mathbb{R}^d \to \mathbb{R}^K \) be the model with parameters \( \theta \). The (population) adversarial loss is defined as
\[
R_f(\theta, \epsilon) := \mathbb{E} \left[ l\left(f_\theta\left[x + A_\epsilon(f_\theta, x, y)\right], y \right) \right],
\]
where \( A_\epsilon \) is an attack of strength \( \epsilon > 0 \) and is designed to worsen the loss in the following way:
\begin{align*}
A_\epsilon(f_\theta, x, y) := \arg\max_{z \in \mathcal{R}(0, \epsilon)} \left\{ l(f_\theta(x + z), y) \right\}.
\label{eq:adv_attack}
\end{align*}

In the above, \( z \) is constrained to the perturbation set \( \mathcal{R}(0,\epsilon) \); throughout this paper we focus on \( \ell_\infty \) attacks, i.e., the \( \ell_\infty \)-ball of radius \( \epsilon \) centered at the origin. We use the \( \ell_\infty \) attack for both adversarial training, calibration-time and test-time attacks in experiments, and our proofs are carried out under \( \ell_2 \) attack assumption. We follow the general adversarial-training framework of \citep{xing2021generalization}:

\begin{algorithm}[!htbp]
\footnotesize
\caption{A General Form of Adversarial Training}\label{alg:AT}
\begin{algorithmic}[1]
\State \textbf{Input:} data $(x_i,y_i)_{i=1}^n$, attack bound $\epsilon$, steps $T$, init $\theta^{(0)}$, step size $\eta$.
\For{$t=1$ \textbf{to} $T$}
  \For{$i=1$ \textbf{to} $n$}
    \State $\tilde{x}_i^{(t-1)} \gets x_i + A_{\epsilon}(f_{\theta^{(t-1)}}, x_i, y_i)$
  \EndFor
  \State $\theta^{(t)} \gets \theta^{(t-1)} - \eta \nabla_{\theta}\,\frac{1}{n}\!\sum_{i=1}^{n}
         \ell\!\big(f_{\theta^{(t-1)}}(\tilde{x}_i^{(t-1)}),\,y_i\big)$
\EndFor
\State \textbf{Output:} $\theta^{(T)}$
\end{algorithmic}
\end{algorithm}

Briefly speaking, in Algorithm \ref{alg:AT}, in each iteration, we first calculate the attack of each sample given the current model, and then update the model's parameters based on the loss value calculated from the attacked samples.

We use \(\ell_\infty\)-FGSM for both test-time attacks and adversarial training because it matches our \(\ell_\infty\) experimental threat model and requires only a single gradient step, enabling scalable, reproducible runs across many \(\epsilon\) values and datasets. Its simplicity keeps training and evaluation aligned, so observed effects are attributable to \(\epsilon_{\mathrm{cal}}\) rather than attack optimization.

\textbf{Norm choice.} We present the theory under \(\ell_{2}\)-bounded perturbations for algebraic clarity. However, Split CP is attack-agnostic for any norm-bounded threat model, so the guarantees do not depend on the specific attack or norm. In experiments we adopt \(\ell_{\infty}\) attacks because, at comparable budgets, \(\ell_{2}\) perturbations leave clean models with high accuracy, revealing little about robustness. \(\ell_{\infty}\) attacks more reliably degrade performance across datasets, thereby stressing the system and making the effects of \(\epsilon_{\mathrm{cal}}\) measurable. Thus, using \(\ell_{2}\) in analysis and \(\ell_{\infty}\) in experiments is consistent with the theory and better suited to our empirical goals.

\subsection{The Whole Pipeline}
After introducing Split CP and adversarial training, we present the whole pipeline as follows:

\begin{itemize}
    \item Training stage: Due to potential corruptions in the testing stage, we use adversarial training with attack strength $\epsilon_{train}$ on the training dataset $I_1$ with size $n_1$ to train the regression/classification model.
    \item Calibration: We inject an adversarial attack in each sample with attack strength $\epsilon_{cal}$ for the calibration data $I_2$  and construct the Split CP given the prediction accuracy guarantee $1-\alpha$.
    \item Testing Stage: The testing sample may suffer from attacks with unknown attack strength $\epsilon_{test}$. Given the input testing sample in the testing dataset $I_3$, we perform Split CP and obtain the prediction set.
\end{itemize}
Our target in this paper is to analyze the impact of adversarial training and adversarial attack towards the Split CP procedure. In particular, we will consider whether the calibration set $I_2$ contains adversarial attack or not, and whether we use clean training or adversarial training in the training stage.

\section{Main Results} 

In this section, we present our main theoretical results on the impact of adversarial attacks to the calibration set within Split CP. We analyze how such attacks influence the validity and efficiency of prediction sets when prediction sets are evaluated on adversarially perturbed test inputs in Theorem \ref{thm:theorem1} and Theorem \ref{thm:theorem2}. Furthermore, we provide a formal proposition that offers a theoretical justification for the observed reduction in prediction set size under adversarial training in training stage in Theorem \ref{thm:theorem3}.

\subsection{Validity of Split CP at the post-training stage}

In the following, we first present a previous result from \citep{oliveira2024split}. The following Proposition \ref{prop:split_bound} provides an error bound on the validity of the split CP with the accuracy of the target prediction coverage $1-\alpha$. Unlike common literature in which only one side of the error bound is presented, \citep{oliveira2024split} provides both the upper bound and the lower bound. This is essential to help us further construct the pair of upper and lower bounds for Split CP generated by adversarially attacked calibration set using adversarial testing data.

To introduce the result of Proposition \ref{prop:split_bound}, we first present some assumptions:

\begin{assumption}\label{assumption}
    There exist constants \(e_{\text{cal}} \in (0, 1)\), \(d_{\text{cal}} \in (0, 1)\), and \(e_{\text{train}} \in (0, 1)\) such that the following holds:

   \noindent \textbf{Calibration Concentration:}
   \begin{eqnarray*}
       &&P\left[\left|\frac{1}{n_{2}}\sum_{(x_{i},y_{i}) \in I_{2}}\mathbf{1}\!\{f_{y_i}(x_i) \le Q\}- P_{q,\text{train}}\right|\le e_{\text{cal}}\right]\\
       &&\ge 1 - d_\text{cal},
   \end{eqnarray*}
    where $(x_{*},y_{*})$ is any other data point independent of training, calibration, and testing data, \( P_{q,\text{train}} = P[f_{y_*}(x_*) \le q_{\text{train}} \mid (x_j, y_j)_{j \in |I_{1}|}] \). The mapping \(q: (X \times Y)^{I_{1}} \to \mathbb{R} \) is any measurable function, and  \(q_{train} = q((x_{i},y_{i}),i\in I_{1})\).
    
\noindent\textbf{Test-Time Stability:}
    \begin{align*}
    &\left|
        P\!\left[f_{x_k}(y_k) \le q_{\text{train}}\right]
        -
        P\!\left[f_{x_*}(y_*) \le q_{\text{train}}\right]
    \right|
    \le e_{\text{train}},\\
    &\text{for } (x_{k},y_{k}) \in I_{3}.
    \end{align*}

\end{assumption}

The conditions in Assumption~\ref{assumption}, following \citep{oliveira2024split}, introduce quantities used to control the error terms in Proposition~\ref{prop:split_bound}. In particular, $d_{cal}$ and $e_{cal}$ are chosen to bound the discrepancy between the empirical CDFs of the nonconformity scores computed on the training and calibration samples, while $e_{train}$ enforces that each test point is independent of the training data.

The following is the error bound statement in \citep{oliveira2024split}:

\begin{proposition}\label{prop:split_bound}
Under Assumption \ref{assumption}, if the sample data $\{(x_{i},y_{i})\}$ for $i\in 1,2,\cdots,n+1$ 
are exchangeable, then the prediction set $C(\cdot)$ given by Split CP satisfies:
\begin{align*}
    |\mathrm{P}(Y_{n+1}\in C(X_{n+1}))-1-\alpha|\le e_{\text{train}}+d_{\text{cal}}+e_{\text{cal}}.
\end{align*}  
\end{proposition}

Given the two-sided error bound result in Proposition \ref{prop:split_bound}, Theorem \ref{thm:theorem1} below analyzes the gap between the prediction accuracy of Split CP and the target accuracy $1-\alpha$. The theorem considers the case where there is a test-time attack with strength $\epsilon_{test}$ and the calibration attack strength is taken as $\epsilon_{cal}$.

\begin{theorem}\label{thm:theorem1}
Assume $f$  the predictive probability distribution of class true label $y$ given observed input $x$ is smooth for a trained neural network. In the case of a neural network, this corresponds to the probability assigned to each class after applying the softmax transformation to the model's output scores. Denote function $g(f_{y}(x))$ as the density of $f_{y}(x)$. Suppose $g$ is second-order differentiable w.r.t. $f_{y}(x)$. Let $C_{\epsilon_{cal}}(x_{test} + \epsilon_{test} A_{test})$ be the prediction set produced by the Split CP algorithm using calibration set with attack strength $\epsilon_{cal}$ given the adversarially attacked test input $x_{test} + \epsilon_{test} A_{test}$. Then for any fixed $\epsilon_{cal}=o(1)$, the coverage of $C_{\epsilon_{cal}}(x_{test} + \epsilon_{test} A_{test}))$ satisfies:
\begin{eqnarray*}
    &&|\mathbb{P}(y_{test}\in C_{\epsilon_{cal}}(x_{test}+\epsilon_{test}A_{test}))-(1-\alpha)|\le\\
    &&e_{train}+d_{cal}+e_{cal}\\
    &&+\left((2\epsilon_{test}-\epsilon_{cal})||\nabla f_{x_{test}}||-\epsilon_{cal}\cdot c \right)\\
    &&g(Q_{1-\alpha}(f_{y_{cal}}(x_{cal})))+o\left(\epsilon_{cal}^{2}+\epsilon_{test}^{2}\right).
\end{eqnarray*}
where we rewrite the attack as $\epsilon_{test}A_{test}$ and 
\begin{eqnarray*}
    A_{test}&\propto&\frac{\partial L(f(x_{test}),y_{test})}{\partial x_{test}}\\
    &=&\frac{\partial L(f(x_{test}),y_{test})}{\partial f_{y_{test}}(x_{test})}\frac{\partial f_{y_{test}}(x_{test})}{\partial x_{test}}
\end{eqnarray*}
with $||A_{test}|| =1$.
\end{theorem}

The proof of Theorem~\ref{thm:theorem1} is in Appendix~\ref{sec:proof-thm1}. A brief intuition is as follows. Under mild regularity conditions—specifically, assuming the smoothness of the probability density function of $g(f_{y}(x))$ and the smoothness of $G(f_{y}(x))$, the prediction-accuracy gap between the prediction set generated from an attacked calibration set and the desired target accuracy is an increasing function of $\epsilon_{cal}$ when the $\epsilon_{test}$ range is held fixed. In particular, with $\epsilon_{test}$ fixed, larger $\epsilon_{cal}$ induces larger prediction sets and hence smaller accuracy shortfall. It is easy to see that when $\epsilon_{cal} = \epsilon_{test}$, the problem reduces to the exchangeable scenario treated in Proposition~\ref{prop:split_bound}.

Beyond establishing this dependence for the calibration split \(I_2\), we also study how to attain a tolerance band around the target when $\epsilon_{test}$ is unknown. Since exact attainment of $1-\alpha$ is impossible without knowing $\epsilon_{test}$, we consider a practical band (e.g., $87\%–93\%$ for a $90\%$ target). Theorem~\ref{thm:theorem2} connects $\epsilon_{cal}$, the tolerance width, the target level, and $\epsilon_{test}$, yielding conditions under which coverage falls within the desired band over a range of test-time attacks. This provides a principled rule for selecting the smallest $\epsilon_{cal}$ that achieves the tolerance across the anticipated $\epsilon_{test}$ range.

\begin{theorem}\label{thm:theorem2}
We assume that $\epsilon_{cal}$ and $\epsilon_{test}$ are $o(1)$. Let $\beta$ be the fixed tolerance for acceptable deviation in the accuracy of the test data. Under the same conditions as in Theorem \ref{thm:theorem1}, the Split CP will give an accuracy guarantee $P(y_{test}\in C(x_{test} + \epsilon_{test} A_{test})) \in [1-\alpha-\beta,1-\alpha+\beta]$ with the following possible test-time attack strength $\epsilon_{test}$:
\begin{align*}
    \epsilon_{test} &> \frac{-\beta - (e_{train}+d_{cal}+e_{cal})}{2||\nabla f_{x_{test}}||g(Q_{1-\alpha}(f_{y_{cal}}(x_{cal}))}\\
    &+\frac{(c+||\nabla f_{x_{test}}||)\epsilon_{cal}}{2||\nabla f_{x_{test}}||} \\[4pt]
    \epsilon_{test} &< \frac{\beta - (e_{train}+d_{cal}+e_{cal})}{2||\nabla f_{x_{test}}||g(Q_{1-\alpha}(f_{y_{cal}}(x_{cal}))}\\
    &+\frac{(c+||\nabla f_{x_{test}}||)\epsilon_{cal}}{2||\nabla f_{x_{test}}||}\\ 
    \intertext{The length of the tolerance interval of $\epsilon_{test}$ is}
    &    \frac{2\beta}{2||\nabla f_{x_{test}}||g(Q_{1-\alpha}(f_{y_{cal}}(x_{cal}))}.
\end{align*}
\end{theorem}

To illustrate Theorem \ref{thm:theorem2}, changing the calibration strength $\epsilon_{cal}$ shifts the entire feasible interval of test-time attacks at which Split CP meets the tolerance band $[\,1-\alpha-\beta,\,1-\alpha+\beta\,]$. Increasing $\epsilon_{cal}$ moves the interval toward larger attack levels; decreasing it moves the interval toward smaller ones. Crucially, this shift does not change the interval’s length, which is governed by the tolerance $ \beta $ and the model’s local sensitivity.

After explaining the impact of adversarial attack in Split CP, we further present the analysis of how adversarial training in the training stage affects the final Split CP result in the following section.

\subsection{Impact of adversarial training at the training stage}

To investigate the impact of adversarial training on the prediction set size given by the Split CP, we first recall a relevant result from the literature. The following Example \ref{ex:example1}, summarizing Theorem~4.3 and  Theorem~4.4 in \citep{li2024adversarial} establishes that, in the presence of test-time adversarial perturbations, adversarially trained models are more robust than those trained on clean data. This guarantee motivates our subsequent analysis of how such robustness translates into improved prediction-set efficiency under Split CP.

\begin{example}\label{ex:example1}
For sufficiently large $d$, consider training a two-layer neural network of the form 
\begin{align*}
F(X) &= \big(F_1(X),F_2(X),\ldots,F_k(X)\big), \\
F_i(X) :&= \sum_{r \in [m]} \sum_{p \in [P]} \widetilde{\mathrm{ReLU}}\!\big(\langle w_{i,r}, x_p \rangle\big),
\end{align*}
where $X=(x_p)_{p \in [P]} \in (\mathbb{R}^d)^P$, $w_{i,r} \in \mathbb{R}^d$ are trainable weights, and $m=\mathrm{polylog}(d)$ denotes the network width. 
Suppose the model is initialized randomly and trained for 
$T = \Theta(\mathrm{poly}(d)/\eta)$ iterations on a sampled training dataset $\mathcal{Z}$ 
using a learning rate $\eta$. Under these conditions, the following properties hold with high probability:
\begin{enumerate}
    \item \textbf{Standard Training:} The model at T iteration $F^{(T)}$  converges to a non-robust global minimum. 
    Let $F_{i}^{T}$ be the $F_{i}(X)$ at T-th iteration. In particular, there exists a perturbation $\Delta(X,y)$ independent of $F^{(T)}$ such that the robust test accuracy is poor:
    \begin{align*}
        \mathbb{P}_{(X,y)\sim \mathcal{D}} 
        &\Big[ \arg\max_{i \in [k]} F^{(T)}_{i}(X+\Delta(X,y)) \neq y \Big] \\[2pt]
        &= 1 - o(1).
    \end{align*}
    \item \textbf{Adversarial Training:} The model $F^{(T)}$ converges to a robust global minimum. 
    Consequently, for perturbations $\|\Delta\|_{\infty} \le \epsilon$, the robust test accuracy is high:
    \begin{align*}
        \mathbb{P}_{(X,y)\sim \mathcal{D}} 
        \Big[ \exists \Delta \in (\mathbb{R}^d)^P, \ \|\Delta\|_{\infty} \le \epsilon 
        \ \text{s.t.}\ \\[-1pt]
        \arg\max_{i \in [k]} F^{(T)}_{i}(X+\Delta) \neq y \Big] 
        &= o(1).
    \end{align*}
\end{enumerate}
\end{example}

Example~\ref{ex:example1} shows that, under specific conditions, adversarial training increases the prediction accuracy in the presence of $\ell_\infty $ test-time attacks. We use this result as intuition and supporting evidence for the analysis in Theorem~\ref{thm:theorem3}.

\begin{theorem}\label{thm:theorem3}
Assume that the predictive distribution $f$ assigns probability mass uniformly across all incorrect classes, while allocating the remaining probability to the true class.  Let $P_{y,\text{true}}$ be the predictive probability for the true label.
Split CP gives the prediction set by $C(x_{n+1}) = \{ y \in \{1, 2, \dots, K\} : f_y(x_{n+1}) \geq Q\}$ as mentioned in Section \ref{sec:split-cp}. The event stands for the number of classes contained in the prediction set given by the cardinality of $C(x_{n+1})$ given by:
\begin{align*}
    |C(x_{test})| = \sum_{i=1}^{k}\mathbbm{1}\left(f_y(x_{test}) \geq 1-Q\right).
\end{align*}
The expectation of the prediction set size can be bounded by:
\begin{align*}
    &&\mathbb{E}(|C(x_{test})|)\le 1-\alpha \pm h(\epsilon_{cal},\epsilon_{test}) \\
    &&+(K-1)\left(1-\frac{1-Q}{1-P_{y,true}}\right)^{K-2}.
\end{align*}

As in the supporting evidence in Example ~\ref{ex:example1}, when adversarial training increases the prediction accuracy in the presence of $\ell_\infty $ test-time attacks, the following holds
\[
\frac{1-Q^{\mathrm{adv}}}{\,1-P_{y,\mathrm{true}}^{\mathrm{adv}}\,}
\;<\;
\frac{1-Q^{\mathrm{clean}}}{\,1-P_{y,\mathrm{true}}^{\mathrm{clean}}\,},
\]
where \(Q^{\mathrm{adv}}\) and \(Q^{\mathrm{clean}}\) are the \((1-\alpha)\)-quantiles of the HPS score on the calibration set for adversarially trained and cleanly trained models, respectively, and \(P_{y,\mathrm{true}}^{\mathrm{adv}}\) and \(P_{y,\mathrm{true}}^{\mathrm{clean}}\) denote the corresponding true-label probabilities. Since the upper bound on the expected set size,
\((K-1)\!\left(1-\frac{1-Q}{1-P_{y,\mathrm{true}}}\right)^{K-2}\),
is increasing in \(\frac{1-Q}{1-P_{y,\mathrm{true}}}\), the adversarially trained model yields a smaller expected prediction-set size than the cleanly trained model.
\end{theorem}

The proof of Theorem~\ref{thm:theorem3} is provided in Appendix~\ref{sec:proof-thm3}. 
At a high level, the argument separates the prediction set into two components: (i) the contribution from the true label, and (ii) the contribution from the remaining (incorrect) labels.  

The assumption in Theorem \ref{thm:theorem3} mirrors the standard symmetric and uniform label-noise model in multiclass learning. Conditional on being wrong, each incorrect class is equally likely which is widely used for tractable and conservative analyses as discussed in \citep{vanrooyen2015unhinged,oyen2022labelnoise}.

For the true label \(y\), Theorem~\ref{thm:theorem1} bounds the deviation of coverage from the nominal level \(1-\alpha\) by a function of the calibration and test perturbations:
\begin{eqnarray*}
&&h(\epsilon_{\mathrm{cal}},\epsilon_{\mathrm{test}}) = e_{train}+d_{cal}+e_{cal}\\
    &&+\left((2\epsilon_{test}-\epsilon_{cal})||\nabla f_{x_{test}}||-\epsilon_{cal}\cdot c +o(\epsilon_{test}^{2}+\epsilon_{cal}^{2})\right)\\
    &&g(Q_{1-\alpha}(f_{y_{cal}}(x_{cal}))+o\left(\epsilon_{cal}^{2}+\epsilon_{test}^{2}\right).
\end{eqnarray*}

Hence, the expected contribution of the true class to \(|C(x_{test})|\) is \(1-\alpha \pm h(\epsilon_{cal},\epsilon_{test})\). Under the uniform-residual assumption, the expected number of incorrect labels included is at most \((K-1)\!\left(1-\frac{1-Q}{\,1-P_{y,true}}\right)^{K-2}\). Summing these two parts yields the stated expectation for the coverage set size. Intuitively, adversarial training tends to increase \(P_{y,true}\), tightening the coverage deviation and lowering the chance that incorrect labels cross $[K]\setminus y$ producing smaller sets than clean training.

\section{Experiments}
In the experiments, we empirically validate the theoretical results presented in the theorems and propositions. Specifically, we expect that the prediction accuracy of Split CP exhibits a monotonic relationship with respect to the attack strength applied to the calibration set in classification tasks for Theorem \ref{thm:theorem1}. Furthermore, for a fixed calibration-time attack strength, we demonstrate that the prediction set coverage remains within a predictable range over an interval of test-time attack strengths  for Theorem \ref{thm:theorem2}. Finally, we show that adversarial training reduces the expected prediction size by increasing the probability of giving the true prediction for Theorem \ref{thm:theorem3}.

\subsection{Experimental Setup}
\paragraph{Dataset}
CIFAR-10, CIFAR-100 \citep{deng2009imagenet}, TinyImageNet \citep{le2015tiny}, and MNIST \citep{deng2012mnist} are used in the experiments. We sampled 60\% data for training, 20\% for calibration, and the remaining 20\% for testing.

\paragraph{Neural network architecture/model}
We conducted image classification experiments using two distinct architectures: ResNet-50d \citep{he2019bag} and Vision Transformers (ViT) \citep{dosovitskiy2020image}. Both models were trained separately under the same experimental setup, and we compared their performance to assess not only predictive accuracy but also the size of the resulting conformal prediction sets. We chose to include ViT because recent work has demonstrated that transformer-based models are generally more powerful and robust than ResNet-style convolutional networks in image classification tasks under adversarial attack \citep{chen2021vision}. This comparison allows us to evaluate how the choice of model architecture influences the validity and efficiency of prediction sets.

\begin{figure}[H]
    \centering
    \includegraphics[width=0.95\columnwidth]{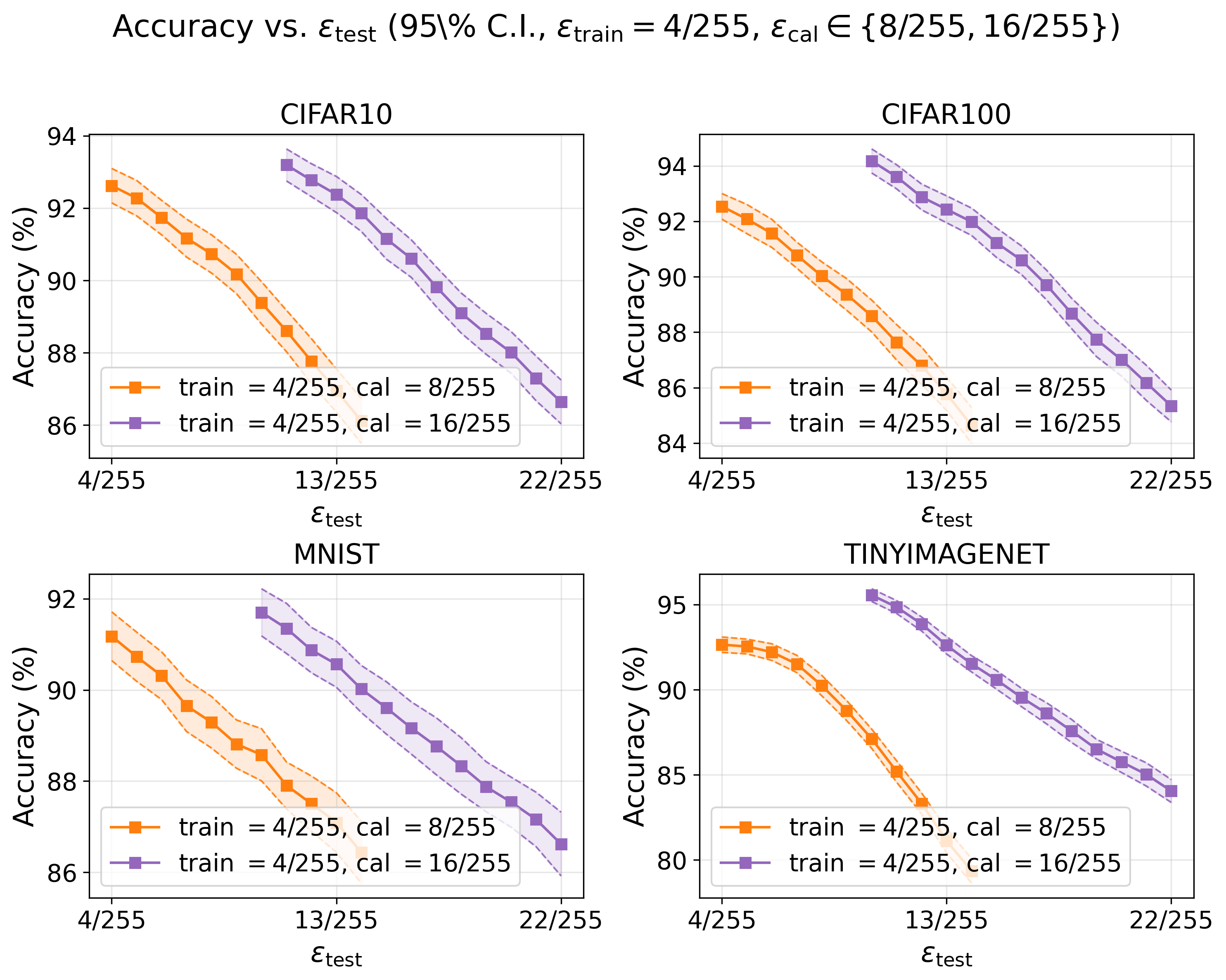}
    \caption{ViT Accuracy}
    \label{fig:vit_accuracy}
\end{figure}

\begin{figure}[H]
    \centering
    \includegraphics[width=0.95\columnwidth]{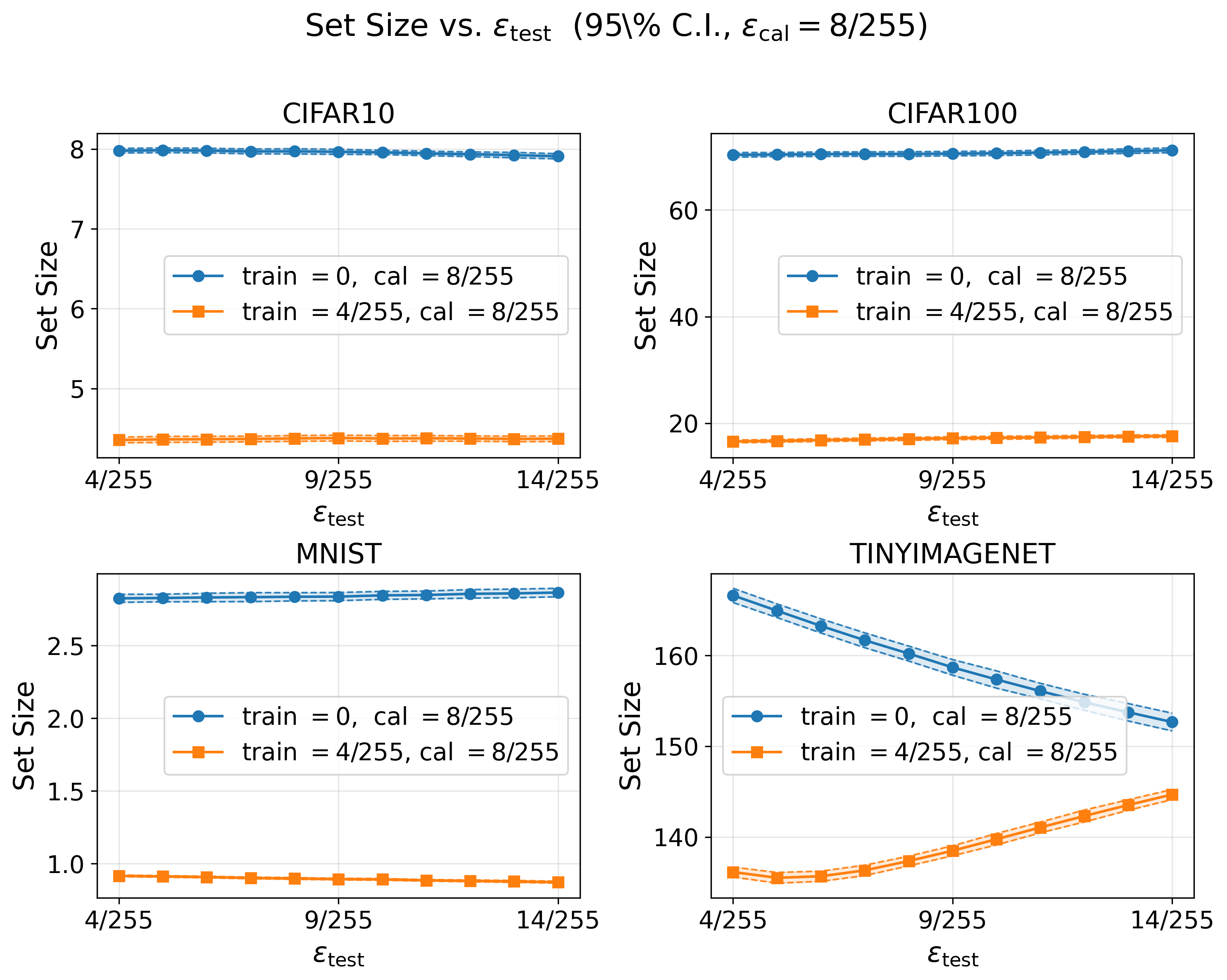}
    \caption{ViT Set Size}
    \label{fig:vit_setsize}
\end{figure}

\paragraph{Training configuration}
For all datasets, we train ResNet50d and ViT to a stable training loss using 40 iterations on a single NVIDIA H200 GPU. Our evaluation focuses on \(\ell_{\infty}\) adversarial attacks with perturbation strengths \(8/255\) and \(16/255\). We do not report \(\ell_{2}\) attacks because the clean (non-adversarially trained) models are comparatively insensitive to \(\ell_{2}\) perturbations, accuracy remains high, thus \(\ell_{2}\) fails to expose meaningful differences in robustness for Split CP. In contrast, \(\ell_{\infty}\) attacks are stronger in our setting and reliably stress the models, providing a more stringent and discriminative robustness assessment for our experiments.

We design a parameter grid with respect to $\epsilon_{train},\epsilon_{cal},\epsilon_{test}$ summarized in Table \ref{tab:experiment-setup-compact} to systematically explore prediction accuracy across five different datasets, targeting a prediction accuracy coverage level of $90\%$.

\begin{table}[htbp]
\centering
\scriptsize
\caption{Adversarial training and testing configurations under different calibration levels.}
\label{tab:experiment-setup-compact}
\begin{tabular}{lc}
\toprule
\textbf{Calibration $\epsilon$} & \textbf{Training / Test $\epsilon$} \\
\midrule
8/255   & Train: \{0, 4, 8, 12, 16\}/255 \newline Test: \{2–14\}/255 \\
16/255  & Train: \{0, 4, 8, 12, 16\}/255 \newline Test: \{10–22\}/255 \\
\bottomrule
\end{tabular}
\end{table}

\paragraph{Evaluation}
We evaluated prediction accuracy and prediction set size across all parameter configurations and datasets. The results from clean training under each level of test-time adversarial attack serve as our baseline for comparison. All models are trained until the training loss reaches a pre-specified target threshold, with the training process typically requiring several days to complete. Besides, while some related works, such as RSCP \citep{gendler2022adversarially}, consider robustness of Split CP, some preliminary trails in the  CIFAR-100 dataset showed that the RSCP method under an $\mathcal{L}_{\infty}$ attack with strength $8/255$ on the test set failed to maintain the desired coverage. Therefore, we focus on comparing the performance of the approach considered in this paper.

\paragraph{Results}
To verify Theorems~\ref{thm:theorem1}–\ref{thm:theorem3}, we run three complementary studies. 
(i) For Theorem~\ref{thm:theorem1}, we assess the Split CP procedure under a calibration-time attack of strength $\epsilon_{cal}$ and confirm that empirical coverage  varies monotonically with $\epsilon_{cal}$. 
(ii) For Theorem~\ref{thm:theorem2}, fixing $\epsilon_{cal}$, we sweep the test-time attack $\epsilon_{test}$ in a neighborhood of $\epsilon_{cal}$ (including $\epsilon_{test} = \epsilon_{cal}$ and check that coverage stays within the prescribed tolerance band around the target level. (iii) For Theorem~\ref{thm:theorem3}, we compare clean training versus adversarial training while holding Split CP fixed, and show that adversarial training yields smaller expected prediction-set sizes at comparable coverage.

We begin with Theorem~\ref{thm:theorem1}: we examine the monotonicity of Split CP’s empirical accuracy as the calibration attack strength \(\epsilon_{\mathrm{cal}}\) varies, holding the \emph{test-time} attack range fixed across datasets. We then turn to Theorem~\ref{thm:theorem2}: fixing \(\epsilon_{\mathrm{cal}}\), we sweep \(\epsilon_{\mathrm{test}}\) to assess robustness under different evaluation conditions. Throughout, we use a tolerance \(\beta=2\%\); thus, for a 90\% target, acceptable coverage/accuracy lies in \([88\%,92\%]\) over a range of \(\epsilon_{\mathrm{test}}\) provided \(\epsilon_{\mathrm{cal}}\) is chosen appropriately.

Across all models and datasets, the results align with Theorems~\ref{thm:theorem1}--\ref{thm:theorem3}. In Fig.~\ref{fig:vit_accuracy}, with adversarial training at \(\epsilon_{\mathrm{train}}=4/255\), test accuracy is nonincreasing in \(\epsilon_{\mathrm{test}}\) for \(\epsilon_{\mathrm{cal}}\in\{8/255,16/255\}\); for fixed \(\epsilon_{\mathrm{test}}\), larger \(\epsilon_{\mathrm{cal}}\) yields higher accuracy (Theorem~\ref{thm:theorem1}). For fixed \(\epsilon_{\mathrm{cal}}\), the \(\epsilon_{\mathrm{test}}\) interval that keeps performance within the 88\%–92\% band widens as \(\epsilon_{\mathrm{cal}}\) increases, while prediction-set sizes remain nearly unchanged (Theorem~\ref{thm:theorem2}). Fig.~\ref{fig:vit_setsize} shows that adversarially trained models produce substantially smaller prediction sets than cleanly trained models on CIFAR-10/100, MNIST and TinyImageNet, consistent with Theorem~\ref{thm:theorem3}. Additional results and the resnet50d counterparts appear in Appendix~\ref{sec:all-experiments}. 

\section{Conclusion}
This work bridges the gap between conformal prediction theory and adversarial robustness, providing both theoretical insights and empirical evidence on how adversarial perturbations influence prediction set validity. By formalizing the relationship between calibration attack strength and test-time robustness, we establish conditions under which target coverage can be maintained within predefined tolerances. Our analysis further reveals that adversarial training not only enhances robustness but also reduces prediction set size through increased output entropy, thereby improving informativeness. Extensive experiments across diverse datasets and architectures confirm the monotonicity of coverage behavior, the feasibility of accuracy-range guarantees, and the efficiency gains from adversarial training. Collectively, these results offer a principled framework for deploying conformal prediction in adversarial environments, paving the way for reliable uncertainty quantification in safety-critical applications. 

\bibliographystyle{apalike}  
\bibliography{citation}

\begin{thebibliography}{}

\bibitem[Aldahdooh et~al., 2021]{aldahdooh2021reveal}
Aldahdooh, A., Hamidouche, W., and Deforges, O. (2021).
\newblock Reveal of vision transformers robustness against adversarial attacks.
\newblock {\em arXiv preprint arXiv:2106.03734}.

\bibitem[Andriushchenko and Flammarion, 2020]{andriushchenko2020understanding}
Andriushchenko, M. and Flammarion, N. (2020).
\newblock Understanding and improving fast adversarial training.
\newblock {\em Advances in Neural Information Processing Systems}, 33:16048--16059.

\bibitem[Areces et~al., 2025]{areces2025onlinecp}
Areces, F., Mohri, C., Hashimoto, T., and Duchi, J.~C. (2025).
\newblock Online conformal prediction via online optimization.
\newblock In {\em International Conference on Machine Learning (ICML)}.

\bibitem[Assion et~al., 2019]{assion2019attack}
Assion, F., Schlicht, P., Gre{\ss}ner, F., Gunther, W., Huger, F., Schmidt, N., and Rasheed, U. (2019).
\newblock The attack generator: A systematic approach towards constructing adversarial attacks.
\newblock In {\em Proceedings of the IEEE/CVF Conference on Computer Vision and Pattern Recognition Workshops}, pages 0--0.

\bibitem[Bai et~al., 2021]{bai2021recent}
Bai, T., Luo, J., Zhao, J., Wen, B., and Wang, Q. (2021).
\newblock Recent advances in adversarial training for adversarial robustness.
\newblock {\em arXiv preprint arXiv:2102.01356}.

\bibitem[Banerji et~al., 2023]{banerji2023clinical}
Banerji, C.~R., Chakraborti, T., Harbron, C., and MacArthur, B.~D. (2023).
\newblock Clinical ai tools must convey predictive uncertainty for each individual patient.
\newblock {\em Nature medicine}, 29(12):2996--2998.

\bibitem[Bao et~al., 2025]{bao2025opsa}
Bao, J., Dang, C., Luo, R., Zhang, H., and Zhou, Z. (2025).
\newblock Enhancing adversarial robustness with conformal prediction: A framework for guaranteed model reliability.
\newblock {\em arXiv preprint arXiv:2506.07804}.

\bibitem[Barber et~al., 2023]{barber2023conformal}
Barber, R.~F., Candes, E.~J., Ramdas, A., and Tibshirani, R.~J. (2023).
\newblock Conformal prediction beyond exchangeability.
\newblock {\em The Annals of Statistics}, 51(2):816--845.

\bibitem[Bastani et~al., 2022]{bastani2022mvp}
Bastani, O., Gupta, V., Jung, C., Noarov, G., Ramalingam, R., and Roth, A. (2022).
\newblock Practical adversarial multivalid conformal prediction.
\newblock In {\em Advances in Neural Information Processing Systems (NeurIPS)}.

\bibitem[Biggio et~al., 2013]{biggio2013evasion}
Biggio, B., Corona, I., Maiorca, D., Nelson, B., {\v{S}}rndi{\'c}, N., Laskov, P., Giacinto, G., and Roli, F. (2013).
\newblock Evasion attacks against machine learning at test time.
\newblock In {\em Joint European conference on machine learning and knowledge discovery in databases}, pages 387--402. Springer.

\bibitem[Candes et~al., 2023]{candes2023conformalized}
Candes, E., Lei, L., and Ren, Z. (2023).
\newblock Conformalized survival analysis.
\newblock {\em Journal of the Royal Statistical Society Series B: Statistical Methodology}, 85(1):24--45.

\bibitem[Carmon et~al., 2019]{carmon2019unlabeled}
Carmon, Y., Raghunathan, A., Schmidt, L., Liang, P., and Duchi, J.~C. (2019).
\newblock Unlabeled data improves adversarial robustness.
\newblock In {\em Advances in Neural Information Processing Systems (NeurIPS)}.

\bibitem[Chen et~al., 2021]{chen2021vision}
Chen, X., Hsieh, C.-J., and Gong, B. (2021).
\newblock When vision transformers outperform resnets without pre-training or strong data augmentations.
\newblock {\em arXiv preprint arXiv:2106.01548}.

\bibitem[Deng et~al., 2009]{deng2009imagenet}
Deng, J., Dong, W., Socher, R., Li, L.-J., Li, K., and Fei-Fei, L. (2009).
\newblock Imagenet: A large-scale hierarchical image database.
\newblock In {\em 2009 IEEE conference on computer vision and pattern recognition}, pages 248--255. Ieee.

\bibitem[Deng, 2012]{deng2012mnist}
Deng, L. (2012).
\newblock The mnist database of handwritten digit images for machine learning research [best of the web].
\newblock {\em IEEE signal processing magazine}, 29(6):141--142.

\bibitem[Dosovitskiy et~al., 2020]{dosovitskiy2020image}
Dosovitskiy, A., Beyer, L., Kolesnikov, A., Weissenborn, D., Zhai, X., Unterthiner, T., Dehghani, M., Minderer, M., Heigold, G., Gelly, S., et~al. (2020).
\newblock An image is worth 16x16 words: Transformers for image recognition at scale.
\newblock {\em arXiv preprint arXiv:2010.11929}.

\bibitem[Gendler et~al., 2022]{gendler2022adversarially}
Gendler, A., Weng, T.-W., Daniel, L., and Romano, Y. (2022).
\newblock Adversarially robust conformal prediction.
\newblock In {\em International Conference on Learning Representations}.

\bibitem[Ghosh et~al., 2023]{ghosh2023prcp}
Ghosh, S., Shi, Y., Belkhouja, T., Yan, Y., Doppa, J.~R., and Jones, B. (2023).
\newblock Probabilistically robust conformal prediction.
\newblock In {\em Proceedings of the 39th Conference on Uncertainty in Artificial Intelligence (UAI)}, volume 216 of {\em Proceedings of Machine Learning Research}, pages 681--690. PMLR.

\bibitem[Goodfellow et~al., 2014]{goodfellow2014explaining}
Goodfellow, I.~J., Shlens, J., and Szegedy, C. (2014).
\newblock Explaining and harnessing adversarial examples.
\newblock {\em arXiv preprint arXiv:1412.6572}.

\bibitem[Goodfellow et~al., 2015]{goodfellow2015explaining}
Goodfellow, I.~J., Shlens, J., and Szegedy, C. (2015).
\newblock Explaining and harnessing adversarial examples.
\newblock In {\em International Conference on Learning Representations (ICLR)}.

\bibitem[Gui et~al., 2024]{gui2024conformal}
Gui, Y., Jin, Y., and Ren, Z. (2024).
\newblock Conformal alignment: Knowing when to trust foundation models with guarantees.
\newblock {\em Advances in Neural Information Processing Systems}, 37:73884--73919.

\bibitem[He et~al., 2019]{he2019bag}
He, T., Zhang, Z., Zhang, H., Zhang, Z., Xie, J., and Li, M. (2019).
\newblock Bag of tricks for image classification with convolutional neural networks.
\newblock In {\em Proceedings of the IEEE/CVF conference on computer vision and pattern recognition}, pages 558--567.

\bibitem[Hullman et~al., 2025]{hullman2025conformal}
Hullman, J., Wu, Y., Xie, D., Guo, Z., and Gelman, A. (2025).
\newblock Conformal prediction and human decision making.
\newblock {\em arXiv preprint arXiv:2503.11709}.

\bibitem[Javanmard et~al., 2020]{javanmard2020precise}
Javanmard, A., Soltanolkotabi, M., and Hassani, H. (2020).
\newblock Precise tradeoffs in adversarial training for linear regression.
\newblock In {\em Conference on Learning Theory}, pages 2034--2078. PMLR.

\bibitem[Kang et~al., 2024]{kang2024colep}
Kang, M., G{\"u}rel, N.~M., Li, L., and Li, B. (2024).
\newblock Colep: Certifiably robust learning-reasoning conformal prediction via probabilistic circuits.
\newblock {\em arXiv preprint arXiv:2403.11348}.

\bibitem[Kiyani et~al., 2025]{kiyani2025decision}
Kiyani, S., Pappas, G., Roth, A., and Hassani, H. (2025).
\newblock Decision theoretic foundations for conformal prediction: Optimal uncertainty quantification for risk-averse agents.
\newblock {\em arXiv preprint arXiv:2502.02561}.

\bibitem[Koh et~al., 2021]{koh2021wilds}
Koh, P.~W., Sagawa, S., Marklund, H., Xie, S.~M., Zhang, M., Balsubramani, A., Hu, W., Yasunaga, M., Phillips, R., et~al. (2021).
\newblock Wilds: A benchmark of in-the-wild distribution shifts.
\newblock In {\em Proceedings of the 38th International Conference on Machine Learning (ICML)}. PMLR.

\bibitem[Kos and Song, 2017]{kos2017delving}
Kos, J. and Song, D. (2017).
\newblock Delving into adversarial attacks on deep policies.
\newblock {\em arXiv preprint arXiv:1705.06452}.

\bibitem[Le and Yang, 2015]{le2015tiny}
Le, Y. and Yang, X. (2015).
\newblock Tiny imagenet visual recognition challenge.
\newblock {\em CS 231N}, 7(7):3.

\bibitem[Lei et~al., 2018]{Lei03072018}
Lei, J., G’Sell, M., Rinaldo, A., Tibshirani, R.~J., and and, L.~W. (2018).
\newblock Distribution-free predictive inference for regression.
\newblock {\em Journal of the American Statistical Association}, 113(523):1094--1111.

\bibitem[Li and Li, 2024]{li2024adversarial}
Li, B. and Li, Y. (2024).
\newblock Adversarial training can provably improve robustness: Theoretical analysis of feature learning process under structured data.
\newblock {\em arXiv preprint arXiv:2410.08503}.

\bibitem[Madry et~al., 2018]{madry2018towards}
Madry, A., Makelov, A., Schmidt, L., Tsipras, D., and Vladu, A. (2018).
\newblock Towards deep learning models resistant to adversarial attacks.
\newblock In {\em International Conference on Learning Representations (ICLR)}.

\bibitem[Massena et~al., 2025]{massena2025lipschitzrcp}
Massena, T., And{\'e}ol, L., Boissin, T., Mamalet, F., Friedrich, C., Serrurier, M., and Gerchinovitz, S. (2025).
\newblock Efficient robust conformal prediction via lipschitz-bounded networks.
\newblock {\em arXiv preprint arXiv:2506.05434}.

\bibitem[Miller et~al., 2020]{miller2020adversarial}
Miller, D.~J., Xiang, Z., and Kesidis, G. (2020).
\newblock Adversarial learning targeting deep neural network classification: A comprehensive review of defenses against attacks.
\newblock {\em Proceedings of the IEEE}, 108(3):402--433.

\bibitem[Moreno-Torres et~al., 2012]{moreno2012unifying}
Moreno-Torres, J.~G., Raeder, T., Alaiz-Rodr{\'\i}guez, R., Chawla, N.~V., and Herrera, F. (2012).
\newblock A unifying view on dataset shift in classification.
\newblock {\em Pattern Recognition}, 45(1):521--530.

\bibitem[Narodytska and Kasiviswanathan, 2017]{narodytska2017simple}
Narodytska, N. and Kasiviswanathan, S.~P. (2017).
\newblock Simple black-box adversarial attacks on deep neural networks.
\newblock In {\em CVPR workshops}, volume~2.

\bibitem[Oliveira et~al., 2024]{oliveira2024split}
Oliveira, R.~I., Orenstein, P., Ramos, T., and Romano, J.~V. (2024).
\newblock Split conformal prediction and non-exchangeable data.
\newblock {\em Journal of Machine Learning Research}, 25(225):1--38.

\bibitem[Oyen et~al., 2022]{oyen2022labelnoise}
Oyen, D., Kucer, M., Hengartner, N., and Singh, H.~S. (2022).
\newblock Robustness to label noise depends on the shape of the noise distribution.
\newblock In {\em Advances in Neural Information Processing Systems}.

\bibitem[Romano et~al., 2020]{romano2020classification}
Romano, Y., Sesia, M., and Candes, E. (2020).
\newblock Classification with valid and adaptive coverage.
\newblock {\em Advances in neural information processing systems}, 33:3581--3591.

\bibitem[Scholten and G{\"u}nnemann, 2025]{scholten2025provably}
Scholten, Y. and G{\"u}nnemann, S. (2025).
\newblock Provably reliable conformal prediction sets in the presence of data poisoning.
\newblock In {\em The Thirteenth International Conference on Learning Representations}.

\bibitem[Shafahi et~al., 2019a]{shafahi2019free}
Shafahi, A., Najibi, M., Ghiasi, A., Xu, Z., Dickerson, J., Studer, C., Davis, L.~S., Taylor, G., and Goldstein, T. (2019a).
\newblock Adversarial training for free!
\newblock In {\em Advances in Neural Information Processing Systems (NeurIPS)}.

\bibitem[Shafahi et~al., 2019b]{shafahi2019adversarial}
Shafahi, A., Najibi, M., Ghiasi, M.~A., Xu, Z., Dickerson, J., Studer, C., Davis, L.~S., Taylor, G., and Goldstein, T. (2019b).
\newblock Adversarial training for free!
\newblock {\em Advances in neural information processing systems}, 32.

\bibitem[Shafahi et~al., 2020]{shafahi2020universal}
Shafahi, A., Najibi, M., Xu, Z., Dickerson, J., Davis, L.~S., and Goldstein, T. (2020).
\newblock Universal adversarial training.
\newblock In {\em Proceedings of the AAAI Conference on Artificial Intelligence}, volume~34, pages 5636--5643.

\bibitem[Shao et~al., 2021]{shao2021adversarial}
Shao, R., Shi, Z., Yi, J., Chen, P.-Y., and Hsieh, C.-J. (2021).
\newblock On the adversarial robustness of vision transformers.
\newblock {\em arXiv preprint arXiv:2103.15670}.

\bibitem[Silva and Najafirad, 2020]{silva2020opportunities}
Silva, S.~H. and Najafirad, P. (2020).
\newblock Opportunities and challenges in deep learning adversarial robustness: A survey.
\newblock {\em arXiv preprint arXiv:2007.00753}.

\bibitem[Sugiyama et~al., 2007]{sugiyama2007covariate}
Sugiyama, M., Krauledat, M., and M{\"u}ller, K.-R. (2007).
\newblock Covariate shift adaptation by importance weighted cross validation.
\newblock {\em Journal of Machine Learning Research}, 8:985--1005.

\bibitem[van Rooyen et~al., 2015]{vanrooyen2015unhinged}
van Rooyen, B., Menon, A.~K., and Williamson, R.~C. (2015).
\newblock Learning with symmetric label noise: The importance of being unhinged.
\newblock In {\em Advances in Neural Information Processing Systems}, volume~28.

\bibitem[Vovk et~al., 2005]{vovk2005algorithmic}
Vovk, V., Gammerman, A., and Shafer, G. (2005).
\newblock {\em Algorithmic Learning in a Random World}.
\newblock Springer.

\bibitem[Wang et~al., 2019]{wang2019improving}
Wang, Y., Zou, D., Yi, J., Bailey, J., Ma, X., and Gu, Q. (2019).
\newblock Improving adversarial robustness requires revisiting misclassified examples.
\newblock In {\em International conference on learning representations}.

\bibitem[Xiao et~al., 2018]{xiao2018generating}
Xiao, C., Li, B., Zhu, J.-Y., He, W., Liu, M., and Song, D. (2018).
\newblock Generating adversarial examples with adversarial networks.
\newblock {\em arXiv preprint arXiv:1801.02610}.

\bibitem[Xie et~al., 2020]{xie2020advprop}
Xie, C., Tan, M., Gong, B., Wang, J., Yuille, A.~L., and Le, Q.~V. (2020).
\newblock Adversarial examples improve image recognition.
\newblock In {\em Proceedings of the IEEE/CVF Conference on Computer Vision and Pattern Recognition (CVPR)}.

\bibitem[Xing, 2023]{xing2024adversarial}
Xing, Y. (2023).
\newblock Adversarial training with generated data in high-dimensional regression: An asymptotic study.

\bibitem[Xing et~al., 2021]{xing2021generalization}
Xing, Y., Song, Q., and Cheng, G. (2021).
\newblock On the generalization properties of adversarial training.
\newblock In {\em International Conference on Artificial Intelligence and Statistics}, pages 505--513. PMLR.

\bibitem[Yan et~al., 2024]{yan2024provablyrobustconformalprediction}
Yan, G., Romano, Y., and Weng, T.-W. (2024).
\newblock Provably robust conformal prediction with improved efficiency.

\bibitem[Zhang et~al., 2019]{zhang2019trades}
Zhang, H., Yu, Y., Jiao, J., Xing, E.~P., El~Ghaoui, L., and Jordan, M.~I. (2019).
\newblock Theoretically principled trade-off between robustness and accuracy.
\newblock In {\em Proceedings of the 36th International Conference on Machine Learning (ICML)}, volume~97 of {\em Proceedings of Machine Learning Research}, pages 7472--7482. PMLR.

\end{thebibliography}

\section*{Checklist}

 \begin{enumerate}

 \item For all models and algorithms presented, check if you include:
 \begin{enumerate}
   \item A clear description of the mathematical setting, assumptions, algorithm, and/or model. [Yes]
   \item An analysis of the properties and complexity (time, space, sample size) of any algorithm. [Yes]
   \item (Optional) Anonymized source code, with specification of all dependencies, including external libraries. [Yes]
 \end{enumerate}

 \item For any theoretical claim, check if you include:
 \begin{enumerate}
   \item Statements of the full set of assumptions of all theoretical results. [Yes]
   \item Complete proofs of all theoretical results. [Yes]
   \item Clear explanations of any assumptions. [Yes]     
 \end{enumerate}

 \item For all figures and tables that present empirical results, check if you include:
 \begin{enumerate}
   \item The code, data, and instructions needed to reproduce the main experimental results (either in the supplemental material or as a URL). [Yes]
   \item All the training details (e.g., data splits, hyperparameters, how they were chosen). [Yes]
         \item A clear definition of the specific measure or statistics and error bars (e.g., with respect to the random seed after running experiments multiple times). [Yes]
         \item A description of the computing infrastructure used. (e.g., type of GPUs, internal cluster, or cloud provider). [Yes]
 \end{enumerate}

 \item If you are using existing assets (e.g., code, data, models) or curating/releasing new assets, check if you include:
 \begin{enumerate}
   \item Citations of the creator If your work uses existing assets. [Not Applicable]
   \item The license information of the assets, if applicable. [Not Applicable]
   \item New assets either in the supplemental material or as a URL, if applicable. [Not Applicable]
   \item Information about consent from data providers/curators. [Not Applicable]
   \item Discussion of sensible content if applicable, e.g., personally identifiable information or offensive content. [Yes] Nvidia-H200
 \end{enumerate}

 \item If you used crowdsourcing or conducted research with human subjects, check if you include:
 \begin{enumerate}
   \item The full text of instructions given to participants and screenshots. [Not Applicable]
   \item Descriptions of potential participant risks, with links to Institutional Review Board (IRB) approvals if applicable. [Not Applicable]
   \item The estimated hourly wage paid to participants and the total amount spent on participant compensation. [Not Applicable]
 \end{enumerate}

 \end{enumerate}

\newpage
\FloatBarrier
\clearpage
\onecolumn
\appendix
\startcontents[app]
\begin{apptocbox}
  \printcontents[app]{l}{1}{\setcounter{tocdepth}{2}}
\end{apptocbox}

\section{Proofs}

\subsection{Lemma 1}\label{sec:lemma1}
\begin{lemma}
Let $f:=f_y(x)\in\mathbb{R}$ be a scalar score and $g:=\|\nabla f\|\ge 0$. 
Let $F_f$ and $f_f$ denote the CDF and PDF of $f$, respectively. 
For $\varepsilon\ge 0$ define $Z_\varepsilon:=f-\varepsilon g$ and 
$q(\varepsilon):=Q_{1-\alpha}(Z_\varepsilon)$, i.e.\ $F_{Z_\varepsilon}(q(\varepsilon))=1-\alpha$.
Assume:
\begin{enumerate}
\item $f_f\big(q(0)\big)>0$ and $F_f$ is differentiable at $q(0):=Q_{1-\alpha}(f)$;
\item the conditional expectation $m(x):=\mathbb{E}[g\,|\,f=x]$ is continuous at $x=q(0)$;
\item $(f,g)$ has a joint density continuous near $(x,\cdot)$ so that differentiation under the integral is valid (e.g., by dominated convergence).
\end{enumerate}
Then, as $\varepsilon\to 0$,
\begin{equation}\label{eq:main-quantile-shift}
Q_{1-\alpha}\!\big(f-\varepsilon g\big)
=Q_{1-\alpha}(f)\;-\;\varepsilon\,\mathbb{E}\!\left[g\,\middle|\,f=Q_{1-\alpha}(f)\right]\;+\;o(\varepsilon).
\end{equation}
If moreover $0\le g\le L$ a.s., then 
\begin{equation}\label{eq:Lipschitz-bound}
\big|\,Q_{1-\alpha}(f-\varepsilon g)-Q_{1-\alpha}(f)\,\big|\le \varepsilon L + o(\varepsilon).
\end{equation}
\end{lemma}

\begin{proof}{Proof of Lemma 1}
\newline
For fixed $x\in\mathbb{R}$,
\begin{align*}
F_{Z_\varepsilon}(x)
&=\mathbb{P}(f-\varepsilon g\le x)
=\iint \mathbf{1}\{u-\varepsilon v\le x\}\,p_{f,g}(u,v)\,du\,dv \\
&=\int\!\left(\int_{-\infty}^{x+\varepsilon v} p_{f,g}(u,v)\,du\right) dv.
\end{align*}
Differentiating w.r.t.\ $\varepsilon$ at $0$ gives
\begin{equation}\label{eq:derivative-cdf}
\left.\frac{\partial}{\partial \varepsilon}F_{Z_\varepsilon}(x)\right|_{\varepsilon=0}
=\int v\,p_{f,g}(x,v)\,dv
=\int v\,p_{g|f}(v\,|\,x)\,p_f(x)\,dv
=f_f(x)\,\mathbb{E}[g\,|\,f=x].
\end{equation}

By definition $F_{Z_\varepsilon}(q(\varepsilon))=1-\alpha$ for all $\varepsilon$.
Differentiate both sides at $\varepsilon=0$ and apply the chain rule:
\[
0=\left.\frac{\partial}{\partial \varepsilon}F_{Z_\varepsilon}(q(\varepsilon))\right|_{\varepsilon=0}
=\underbrace{\left.\frac{\partial}{\partial \varepsilon}F_{Z_\varepsilon}(x)\right|_{\varepsilon=0}}_{\text{use \eqref{eq:derivative-cdf} at }x=q(0)}
+\underbrace{\left.\frac{\partial}{\partial x}F_{Z_0}(x)\right|_{x=q(0)}}_{=\,f_f(q(0))}\,q'(0).
\]
Hence $0=f_f(q(0))\,\mathbb{E}[g\,|\,f=q(0)]+f_f(q(0))\,q'(0)$, so
$q'(0)=-\mathbb{E}[g\,|\,f=q(0)]$.
A first-order Taylor expansion of $q(\varepsilon)$ at $0$ yields \eqref{eq:main-quantile-shift}.
For \eqref{eq:Lipschitz-bound}, if $0\le g\le L$ a.s., then
$f-\varepsilon L \le f-\varepsilon g \le f$ a.s.; monotonicity of quantiles gives
$Q_{1-\alpha}(f-\varepsilon L)\le Q_{1-\alpha}(f-\varepsilon g)\le Q_{1-\alpha}(f)$,
implying the claimed bound.
\end{proof}

\subsection{Proof of Theorem 1}\label{sec:proof-thm1}
\emph{Proof of theorem 1}:
We are considering the non-comformity score HPS $S(x,y) = 1- f_{y}(x)$ where $f$ is the trained model and $f_{y}(x)$ means the probability for the model to make the correct prediction given observation $x$. Now suppose we have test-time attack denoted as $\epsilon_{test}$, and $||A_{test}|| \le 1$ which is the vector indicating the direction of attack.
Thus we have the attacked non-comformity score by taking taylor expansion with respect to $\epsilon_{test}$ as:
\begin{eqnarray*}
    S(\tilde{x}_{test})&&= 1 - f_{y_{test}}(\tilde{x}_{test}) \\
                 &&= 1 - f_{y_{test}}(x_{test} + \epsilon_{test} A_{test}^\top) \\
                 &&= 1 - \left( f_{y_{test}}(x_{test}) + \epsilon_{test} A_{test}^\top \nabla f + o(\epsilon_{test}^2) \right)
\end{eqnarray*}
Now based on the expansion of the prediction accuracy, we can further expand the probability for a test data label with in the prediction set generated by the $\delta_{2}$ attacked calibration set as:
\begin{eqnarray*}
    &&\mathbb{P}\left(y_{test} \in C_{\epsilon_{cal}}(x_{test}+\epsilon_{test} A_{test})\right)\\
    &=&\mathbb{P}\left(S(x_{test}+\epsilon_{test} A_{test},y_{test})\le Q_{1-\alpha}\left(S(x_{cal}+\epsilon_{cal} A_{cal},y_{cal})\right)\right)\\
    &=&\mathbb{P}\left(1 - \left( f_{y_{test}}(x_{test}) + \epsilon_{test} A_{test}^\top \nabla f_{x_{test}} + o(\epsilon_{test}^2) \right) \le Q_{1-\alpha}\left(1 - \left( f_{y_{cal}}(x_{cal}) + \epsilon_{cal} A_{cal}^\top \nabla f_{x_{cal}} + o(\epsilon_{cal}^2) \right)\right)\right)\\
    &=&\mathbb{P}\left( f_{y_{test}}(x_{test}) + \epsilon_{test} A_{test}^\top \nabla f_{x_{test}} + o(\epsilon_{test}^2)\ge Q_{1-\alpha}\left( f_{y_{cal}}(x_{cal}) + \epsilon_{cal} A_{i}^\top \nabla f_{x_{cal}} + o(\epsilon_{cal}^2) \right)\right)\\
    &=&\mathbb{P}\left( f_{y_{test}}(x_{test}) + (\epsilon_{test} -\epsilon_{cal}+\epsilon_{cal})A_{test}^\top \nabla f_{x_{test}} + o(\epsilon_{test}^2)\ge Q_{1-\alpha}\left( f_{y_{cal}}(x_{cal}) + \epsilon_{cal}A_{cal}^\top \nabla f_{x_{cal}} \right)\right)\\
    &=&\mathbb{P}\left( f_{y_{test}}(x_{test}) + \epsilon_{cal} A_{test}^\top \nabla f_{x_{test}} + o(\epsilon_{test}^2+\epsilon_{cal}^2)+(\epsilon_{test} -\epsilon_{cal})A_{test}^\top \nabla f_{x_{test}}
    \ge Q_{1-\alpha}\left( f_{y_{cal}}(x_{cal}) + \epsilon_{cal}A_{cal}^\top \nabla f_{x_{cal}} \right)\right).
\end{eqnarray*}

For $A$, by the definition of adversarial training, we have
\begin{eqnarray*}
    A_{test}\propto\frac{\partial L(f(x_{test}),y_{test})}{\partial x_{test}}=\frac{\partial L(f(x_{test}),y_{test})}{\partial f_{y_{test}}(x_{test})}\frac{\partial f_{y_{test}}(x_{test})}{\partial x_{test}},
\end{eqnarray*}
and $\|A_{test}\|=1$. 

Since $L$ is a loss function, we generally assume that 
\begin{eqnarray*}
    \frac{\partial L(f(x_{test}),y_{test})}{\partial f_{y_{test}}(x_{test})}<0.
\end{eqnarray*}

Thus,
\begin{eqnarray*}
   A_{test}^{\top} \nabla f_{x_{test}} =-\|\nabla f_{x_{test}}\|
\end{eqnarray*}

Denote $\mathbb{E}\!\left[ \|\nabla f_{x_{test}}\|\,\middle|\,f_{y_{test}}(x_{test})=Q_{1-\alpha}(f_{y_{cal}}(x_{cal}))\right]=c$. In this case, based on Lemma \ref{sec:lemma1}, the probability can be rewritten as:
\begin{eqnarray*}
&&\mathbb{P}\left( f_{y_{test}(x_{test})}\ge (2\epsilon_{test}-\epsilon_{cal})||\nabla f_{x_{test}}||-\epsilon_{cal}\cdot c +o(\epsilon_{test}^{2}+\epsilon_{cal}^{2})+Q_{1-\alpha}(f_{y_{cal}}(x_{cal}))\right)
\end{eqnarray*}

We divide the real case into two scenarios:
\begin{enumerate}
    \item $\epsilon_{test}>\epsilon_{cal}$:
\begin{eqnarray*}
    &&\mathbb{P}\left( f_{y_{test}(x_{test})}\ge (2\epsilon_{test}-\epsilon_{cal})||\nabla f_{x_{test}}||-\epsilon_{cal}\cdot c +o(\epsilon_{test}^{2}+\epsilon_{cal}^{2})+Q_{1-\alpha}(f_{y_{cal}}(x_{cal}))\right)\\
    &=&\mathbb{P}\left( f_{y_{test}(x_{test})}\ge Q_{1-\alpha}(f_{y_{cal}}(x_{cal}))\right)\\
    &-&\mathbb{P}\left((2\epsilon_{test}-\epsilon_{cal})||\nabla f_{x_{test}}||-\epsilon_{cal}\cdot c +o(\epsilon_{test}^{2}+\epsilon_{cal}^{2})+Q_{1-\alpha}(f_{y_{cal}}(x_{cal}))\ge f_{y_{test}}(x_{test})\ge Q_{1-\alpha}(f_{y_{cal}}(x_{cal}))\right)\\
    &=&1-\alpha +e_{train}+d_{cal}+e_{cal}\\
    &-&\mathbb{P}\left((2\epsilon_{test}-\epsilon_{cal})||\nabla f_{x_{test}}||-\epsilon_{cal}\cdot c +o(\epsilon_{test}^{2}+\epsilon_{cal}^{2})+Q_{1-\alpha}(f_{y_{cal}}(x_{cal}))\ge f_{y_{test}}(x_{test})\ge Q_{1-\alpha}(f_{y_{cal}}(x_{cal}))\right)
\end{eqnarray*}
    \item $\epsilon_{cal}> \epsilon_{test}$:
\begin{eqnarray*}
    &&\mathbb{P}\left( f_{y_{test}(x_{test})}\ge (2\epsilon_{test}-\epsilon_{cal})||\nabla f_{x_{test}}||-\epsilon_{cal}\cdot c +o(\epsilon_{test}^{2}+\epsilon_{cal}^{2})+Q_{1-\alpha}(f_{y_{cal}}(x_{cal}))\right)\\
    &=&\mathbb{P}\left( f_{y_{test}(x_{test})}\ge Q_{1-\alpha}(f_{y_{cal}}(x_{cal}))\right)\\
    &+&\mathbb{P}\left((2\epsilon_{test}-\epsilon_{cal})||\nabla f_{x_{test}}||-\epsilon_{cal}\cdot c +o(\epsilon_{test}^{2}+\epsilon_{cal}^{2})+Q_{1-\alpha}(f_{y_{cal}}(x_{cal}))\le f_{y_{test}}(x_{test})\le Q_{1-\alpha}(f_{y_{cal}}(x_{cal}))\right)\\
    &=&1-\alpha +e_{train}+d_{cal}+e_{cal}\\
    &+&\mathbb{P}\left((2\epsilon_{test}-\epsilon_{cal})||\nabla f_{x_{test}}||-\epsilon_{cal}\cdot c +o(\epsilon_{test}^{2}+\epsilon_{cal}^{2})+Q_{1-\alpha}(f_{y_{cal}}(x_{cal}))\le f_{y_{test}}(x_{test})\le Q_{1-\alpha}(f_{y_{cal}}(x_{cal}))\right)
\end{eqnarray*}
\end{enumerate}
Due to the assumption of smoothness of the distribution function $G(f_{y}(x))$, we can estimate the probabilities as:
\begin{eqnarray*}
    &&\mathbb{P}\left((2\epsilon_{test}-\epsilon_{cal})||\nabla f_{x_{test}}||-\epsilon_{cal}\cdot c +o(\epsilon_{test}^{2}+\epsilon_{cal}^{2})+Q_{1-\alpha}(f_{y_{cal}}(x_{cal}))\le f_{y_{test}}(x_{test})\le Q_{1-\alpha}(f_{y_{cal}}(x_{cal}))\right)\\
    &=&\left((2\epsilon_{test}-\epsilon_{cal})||\nabla f_{x_{test}}||-\epsilon_{cal}\cdot c +o(\epsilon_{test}^{2}+\epsilon_{cal}^{2})\right)g(Q_{1-\alpha}(f_{y_{cal}}(x_{cal})))\\
    &+&o\left(\left((2\epsilon_{test}-\epsilon_{cal})||\nabla f_{x_{test}}||-\epsilon_{cal}\cdot c+o(\epsilon_{test}^{2}+\epsilon_{cal}^{2})\right)^{2}g^{2}(f_{y}(x) = Q_{1-\alpha}(f_{y_{cal}}(x_{cal}))\right)\\
    &=&\left((2\epsilon_{test}-\epsilon_{cal})||\nabla f_{x_{test}}||-\epsilon_{cal}\cdot c +o(\epsilon_{test}^{2}+\epsilon_{cal}^{2})\right)g( Q_{1-\alpha}(f_{y_{cal}}(x_{cal})))+o\left(\epsilon_{cal}^{2}+\epsilon_{test}^{2}\right)\\
\end{eqnarray*}

Thus, we can conclude the true coverage of $C_{\epsilon_{cal}}(x_{test}+\epsilon_{test}A_{test})$ is:
\begin{eqnarray*}
    &&|\mathbb{P}(y_{test}\in C(x_{test}+\epsilon_{test}A_{test}))-(1-\alpha)|\le e_{train}+d_{cal}+e_{cal}\\
    &+&\left((2\epsilon_{test}-\epsilon_{cal})||\nabla f_{x_{test}}||-\epsilon_{cal}\cdot c\right)g(f_{y}(x) = Q_{1-\alpha}(f_{y_{cal}}(x_{cal}))+o\left(\epsilon_{cal}^{2}+\epsilon_{test}^{2}\right)
\end{eqnarray*}

\subsection{Proof of Theorem 2}\label{sec:proof-thm2}
\emph{Proof of theorem 2}:\\
Assume $\beta$ is a fixed small value (e.g., 2\%).$\beta\asymp \epsilon_{test}\asymp \epsilon_{cal}$. As proved in Theorem \ref{thm:theorem1} we can have:
\begin{eqnarray*}
    &&|\mathbb{P}(y_{test}\in C(x_{test}+\epsilon_{test}A_{test}))-(1-\alpha)|\le e_{train}+d_{cal}+e_{cal}+\left((2\epsilon_{test}-\epsilon_{cal})||\nabla f_{x_{test}}||-\epsilon_{cal}\cdot c +o(\epsilon_{test}^{2}+\epsilon_{cal}^{2})\right)g(Q_{1-\alpha}(f_{y_{cal}}\\
    &&(x_{cal}))+o\left(\epsilon_{cal}^{2}+\epsilon_{test}^{2}\right)
\end{eqnarray*}
Given $\beta$,the test attack strength we can attain the prediction accuracy inside $[1-\alpha-\beta,1-\alpha+\beta]$ given $\epsilon_{cal}$ is:
\begin{eqnarray*}
    &&-\beta < e_{train}+d_{cal}+e_{cal}+\left((2\epsilon_{test}-\epsilon_{cal})||\nabla f_{x_{test}}||-\epsilon_{cal}\cdot c +o(\epsilon_{test}^{2}+\epsilon_{cal}^{2})\right)g(Q_{1-\alpha}(f_{y_{cal}}(x_{cal})) <\beta\\
    &&\Leftrightarrow -\beta - (e_{train}+d_{cal}+e_{cal})<\left((2\epsilon_{test}-\epsilon_{cal})||\nabla f_{x_{test}}||-\epsilon_{cal}\cdot c\right)g(Q_{1-\alpha}(f_{y_{cal}}(x_{cal}))<\beta\\
    &&-(e_{train}+d_{cal}+e_{cal})\\
    &&\Leftrightarrow \frac{-\beta - (e_{train}+d_{cal}+e_{cal})}{g(Q_{1-\alpha}(f_{y_{cal}}(x_{cal}))} < (2\epsilon_{test}\|\nabla f_{test}\|-(c+||\nabla f_{x_{test}}||)\epsilon_{cal})<\frac{\beta + (e_{train}+d_{cal}+e_{cal})}{g( Q_{1-\alpha}(f_{y_{cal}}(x_{cal}))}\\
    &&\Leftrightarrow \frac{-\beta - (e_{train}+d_{cal}+e_{cal})}{2||\nabla f_{x_{test}}||g(Q_{1-\alpha}(f_{y_{cal}}(x_{cal}))}+\frac{(c+||\nabla f_{x_{test}}||)\epsilon_{cal}}{2||\nabla f_{x_{test}}||} < \epsilon_{test}<\frac{\beta - (e_{train}+d_{cal}+e_{cal})}{2||\nabla f_{x_{test}}||g(Q_{1-\alpha}(f_{y_{cal}}(x_{cal}))}\\
    &&+\frac{(c+||\nabla f_{x_{test}}||)\epsilon_{cal}}{2||\nabla f_{x_{test}}||}\\
\end{eqnarray*}
In this case, the length of the tolerence set for $\epsilon_{test}$ is:
\begin{eqnarray*}
    \frac{2\beta}{2||\nabla f_{x_{test}}||g(Q_{1-\alpha}(f_{y_{cal}}(x_{cal}))}
\end{eqnarray*}

\subsection{Proof of Theorem 3}\label{sec:proof-thm3}
\emph{Proof of theorem 3}:\\
The expected number of classes included in the prediction set given by Split CP is:
\begin{eqnarray*}     
&& \mathbb{E}\!\left( \sum_{y=1}^{k} \mathbbm{1}\!\left( f_y(x_{test}) \geq 1-Q \right) \right) \\
&&= \mathbb{E}\!\left( \sum_{y \ne y_{true}}^{k} \mathbbm{1}\!\left( f_y(x_{test}) \geq 1-Q \right) \right) 
   + P\!\left( f_{y_{true}}(x_{n+1}) \geq 1-Q \right) \\[2pt]
\end{eqnarray*}
Based on Theorem \ref{thm:theorem1}, for Split CP we can have the prediction accuracy guarantee for the true class with probability $1-\alpha \pm o(\epsilon_{cal})$.
Meanwhile for the other $K-1$ classes not including the true one, the predictive probability distribution given by $\hat{f}$ has the probability $1-P_{y,\text{true}}$ to make a wrong prediction.
Since the threshold $Q$ given by Split CP is derived by selecting the $(1-\alpha)(1+1/|I_{2}|)$-quantile of the noncomformity score $1-\hat{f}_{y_{true}}(x)$, where $(x,y)$ are from the calibration set.
We can bound the expected set size of Split CP by:
\begin{eqnarray*}
&& \sum_{y \ne y_{true}} P\!\left( f_y(x_{test}) \geq 1-Q \right) 
   + P\!\left( f_{y_{true}}(x_{test}) \geq 1-Q \right) \\[1pt]
&& \le 1 - \alpha \;\pm\; o(\epsilon_{cal})
   + \sum_{y \ne y_{true}} P\!\left( f_y(x_{test}) \geq 1-Q \right) \\[1pt]
\end{eqnarray*}

If we assume the chance for the predictive model $f$ assign probabilities to each mismatched classes uniformly, then the event $(f_{y=1}(x_{n+1}) \geq 1-Q, f_{y=2}(x_{n+1}) \geq 1-Q,\cdots, f_{y=K}(x_{n+1})_{} \geq 1-Q)_{y\ne y_{true}} \sim (1-P_{y_{true}})\mathrm{Dirichlet}(\underbrace{1,1,\ldots,1}_{K-1})$. Due to the symmetric of this distribution, for one single class:
\begin{eqnarray*}
&& P\!\left( f_{y}(x_{test}) \ge 1-Q \right) \\[1pt]
&&= P\!\left( \frac{f_{y}(x_{test})}{\,1-P_{y,true}\,} \;\ge\; \frac{1-Q}{\,1-P_{y,true}\,} \right) \\[1pt]
&&= P\!\left( \mathrm{Beta}(1, K-2) \;\ge\; \frac{1-Q}{\,1-P_{y,true}\,} \right) \\[1pt]
&&=\left( 1 - \frac{1-Q}{\,1-P_{y,true}\,} \right)^{K-2}
\end{eqnarray*}

Thus, the expected set size by Split CP is bounded by:
\begin{eqnarray*}
\mathbb{E}\!\left( C(x_{test}) \right) 
&\le& 1 - \alpha \;\pm\; o(\epsilon_{cal}) 
\;+\; (K-1)\!\left( 1 - \frac{1-Q}{\,1-P_{y,\text{true}}\,} \right)^{K-2}
\end{eqnarray*}

As stated in Example \ref{ex:example1}, adversarial training will lead $P_{y,{true}}$ larger compared to clean training. The following part focuses on proving that $(1-Q)/(1-P_{y_{true}})$ will be smaller under clean training compared to adversarial training.

When there exists adversarial attack on calibration set, we'd expect for all the data in calibration set adversarial training will lead to smaller nonconformity scores $S(x,y)$ make the $(1-\alpha)(1+1/|I_{2}|)$ quantile smaller. In general, we have
\begin{eqnarray*}
    \frac{1-Q^{adv}}{1-P_{y,\text{true}}^{adv}} <  \frac{1-Q^{clean}}{1-P_{y,\text{true}}^{clean}}
\end{eqnarray*}
Thus adversarial training will make the expected set size smaller compared to clean training.

\section{Experiments results}\label{sec:all-experiments}

\subsection{Numerical results}

\begingroup
\small
\setlength\tabcolsep{3pt}
\setlength\LTleft{\fill}\setlength\LTright{\fill}
\begin{longtable}{rrrrrrr}
\caption{ViT-MAE results on CIFAR10.}
\label{tab:vitmae_cifar10_pandas_long}\\
\toprule
 Delta1 &  Delta2 &  Delta3 &  accuracy\_mean &  accuracy\_upper &  accuracy\_lower &  setsize\_mean \\
\midrule
\endfirsthead
\caption[]{ViT-MAE results on CIFAR10.} \\
\toprule
 Delta1 &  Delta2 &  Delta3 &  accuracy\_mean &  accuracy\_upper &  accuracy\_lower &  setsize\_mean \\
\midrule
\endhead
\midrule
\multicolumn{7}{r}{{Continued on next page}} \\
\midrule
\endfoot

\bottomrule
\endlastfoot
  0.000 &   0.031 &   0.008 &         92.976 &          93.451 &          92.501 &         7.965 \\
  0.000 &   0.031 &   0.012 &         92.479 &          92.968 &          91.990 &         7.973 \\
  0.000 &   0.031 &   0.016 &         91.943 &          92.442 &          91.445 &         7.978 \\
  0.000 &   0.031 &   0.020 &         91.189 &          91.671 &          90.707 &         7.982 \\
  0.000 &   0.031 &   0.024 &         90.815 &          91.314 &          90.316 &         7.979 \\
  0.000 &   0.031 &   0.027 &         89.940 &          90.477 &          89.403 &         7.969 \\
  0.000 &   0.031 &   0.031 &         89.298 &          89.848 &          88.748 &         7.970 \\
  0.000 &   0.031 &   0.035 &         88.683 &          89.266 &          88.099 &         7.964 \\
\end{longtable}

\endgroup

\begingroup
\small
\setlength\tabcolsep{3pt}
\setlength\LTleft{\fill}\setlength\LTright{\fill}
\begin{longtable}{rrrrrrr}
\caption{ViT-MAE results on CIFAR100.}
\label{tab:vitmae_cifar100_pandas_long}\\
\toprule
 Delta1 &  Delta2 &  Delta3 &  accuracy\_mean &  accuracy\_upper &  accuracy\_lower &  setsize\_mean \\
\midrule
\endfirsthead
\caption[]{ViT-MAE results on CIFAR100.} \\
\toprule
 Delta1 &  Delta2 &  Delta3 &  accuracy\_mean &  accuracy\_upper &  accuracy\_lower &  setsize\_mean \\
\midrule
\endhead
\midrule
\multicolumn{7}{r}{{Continued on next page}} \\
\midrule
\endfoot

\bottomrule
\endlastfoot
  0.000 &   0.031 &   0.008 &         95.982 &          96.336 &          95.628 &        70.129 \\
  0.000 &   0.031 &   0.012 &         95.229 &          95.599 &          94.858 &        70.285 \\
  0.000 &   0.031 &   0.016 &         94.504 &          94.929 &          94.079 &        70.342 \\
  0.000 &   0.031 &   0.020 &         93.707 &          94.135 &          93.280 &        70.396 \\
  0.000 &   0.031 &   0.024 &         92.693 &          93.202 &          92.184 &        70.467 \\
  0.000 &   0.031 &   0.027 &         91.918 &          92.447 &          91.390 &        70.492 \\
  0.000 &   0.031 &   0.031 &         91.145 &          91.639 &          90.650 &        70.497 \\
  0.000 &   0.031 &   0.035 &         90.329 &          90.862 &          89.796 &        70.571 \\
\end{longtable}

\endgroup

\begingroup
\small
\setlength\tabcolsep{3pt}
\setlength\LTleft{\fill}\setlength\LTright{\fill}
\begin{longtable}{rrrrrrr}
\caption{ViT-MAE results on MNIST.}
\label{tab:vitmae_mnist_pandas_long}\\
\toprule
 Delta1 &  Delta2 &  Delta3 &  accuracy\_mean &  accuracy\_upper &  accuracy\_lower &  setsize\_mean \\
\midrule
\endfirsthead
\caption[]{ViT-MAE results on MNIST.} \\
\toprule
 Delta1 &  Delta2 &  Delta3 &  accuracy\_mean &  accuracy\_upper &  accuracy\_lower &  setsize\_mean \\
\midrule
\endhead
\midrule
\multicolumn{7}{r}{{Continued on next page}} \\
\midrule
\endfoot

\bottomrule
\endlastfoot
  0.000 &   0.031 &   0.008 &         91.758 &          92.259 &          91.256 &         2.815 \\
  0.000 &   0.031 &   0.012 &         91.548 &          92.039 &          91.056 &         2.821 \\
  0.000 &   0.031 &   0.016 &         91.279 &          91.805 &          90.752 &         2.825 \\
  0.000 &   0.031 &   0.020 &         91.125 &          91.590 &          90.659 &         2.827 \\
  0.000 &   0.031 &   0.024 &         91.004 &          91.535 &          90.472 &         2.831 \\
  0.000 &   0.031 &   0.027 &         90.726 &          91.317 &          90.135 &         2.833 \\
  0.000 &   0.031 &   0.031 &         90.460 &          90.998 &          89.922 &         2.836 \\
  0.000 &   0.031 &   0.035 &         90.111 &          90.660 &          89.563 &         2.837 \\
\end{longtable}

\endgroup

\begingroup
\small
\setlength\tabcolsep{3pt}
\setlength\LTleft{\fill}\setlength\LTright{\fill}
\begin{longtable}{rrrrrrr}
\caption{ViT-MAE results on TINYIMAGENET.}
\label{tab:vitmae_tinyimagenet_pandas_long}\\
\toprule
 Delta1 &  Delta2 &  Delta3 &  accuracy\_mean &  accuracy\_upper &  accuracy\_lower &  setsize\_mean \\
\midrule
\endfirsthead
\caption[]{ViT-MAE results on TINYIMAGENET.} \\
\toprule
 Delta1 &  Delta2 &  Delta3 &  accuracy\_mean &  accuracy\_upper &  accuracy\_lower &  setsize\_mean \\
\midrule
\endhead
\midrule
\multicolumn{7}{r}{{Continued on next page}} \\
\midrule
\endfoot

\bottomrule
\endlastfoot
  0.000 &   0.031 &   0.008 &         96.487 &          96.820 &          96.154 &       170.019 \\
  0.000 &   0.031 &   0.012 &         95.553 &          95.885 &          95.221 &       168.337 \\
  0.000 &   0.031 &   0.016 &         94.815 &          95.233 &          94.397 &       166.617 \\
  0.000 &   0.031 &   0.020 &         93.665 &          94.088 &          93.243 &       164.929 \\
  0.000 &   0.031 &   0.024 &         92.283 &          92.804 &          91.761 &       163.257 \\
  0.000 &   0.031 &   0.027 &         91.099 &          91.625 &          90.573 &       161.673 \\
  0.000 &   0.031 &   0.031 &         89.778 &          90.325 &          89.230 &       160.208 \\
  0.000 &   0.031 &   0.035 &         88.793 &          89.373 &          88.213 &       158.684 \\
\end{longtable}

\endgroup

\begingroup
\small
\setlength\tabcolsep{3pt}
\setlength\LTleft{\fill}\setlength\LTright{\fill}
\setlength{\LTcapwidth}{\linewidth}
\begin{longtable}{rrrrrrr}
\caption{ResNet50D results on CIFAR10.}
\label{tab:r50d_cifar10_long}\\
\toprule
 Delta1 &  Delta2 &  Delta3 &  accuracy\_mean &  accuracy\_upper &  accuracy\_lower &  setsize\_mean \\
\midrule
\endfirsthead
\caption[]{ResNet50D results on CIFAR10.} \\
\toprule
 Delta1 &  Delta2 &  Delta3 &  accuracy\_mean &  accuracy\_upper &  accuracy\_lower &  setsize\_mean \\
\midrule
\endhead
\midrule
\multicolumn{7}{r}{{Continued on next page}} \\
\midrule
\endfoot

\bottomrule
\endlastfoot
  0.000 &   0.031 &   0.008 &         94.391 &          94.786 &          93.997 &         6.925 \\
  0.000 &   0.031 &   0.012 &         93.544 &          93.997 &          93.091 &         6.949 \\
  0.000 &   0.031 &   0.016 &         92.637 &          93.104 &          92.171 &         6.958 \\
  0.000 &   0.031 &   0.020 &         91.939 &          92.418 &          91.459 &         6.956 \\
  0.000 &   0.031 &   0.024 &         91.068 &          91.615 &          90.520 &         6.957 \\
  0.000 &   0.031 &   0.027 &         90.240 &          90.779 &          89.702 &         6.954 \\
  0.000 &   0.031 &   0.031 &         89.463 &          90.079 &          88.847 &         6.951 \\
  0.000 &   0.031 &   0.035 &         88.792 &          89.439 &          88.145 &         6.946 \\
\end{longtable}

\endgroup

\begingroup
\small
\setlength\tabcolsep{3pt}
\setlength\LTleft{\fill}\setlength\LTright{\fill}
\setlength{\LTcapwidth}{\linewidth}
\begin{longtable}{rrrrrrr}
\caption{ResNet50D results on CIFAR100.}
\label{tab:r50d_cifar100_long}\\
\toprule
 Delta1 &  Delta2 &  Delta3 &  accuracy\_mean &  accuracy\_upper &  accuracy\_lower &  setsize\_mean \\
\midrule
\endfirsthead
\caption[]{ResNet50D results on CIFAR100.} \\
\toprule
 Delta1 &  Delta2 &  Delta3 &  accuracy\_mean &  accuracy\_upper &  accuracy\_lower &  setsize\_mean \\
\midrule
\endhead
\midrule
\multicolumn{7}{r}{{Continued on next page}} \\
\midrule
\endfoot

\bottomrule
\endlastfoot
  0.000 &   0.031 &   0.008 &         96.849 &          97.154 &          96.544 &        82.209 \\
  0.000 &   0.031 &   0.012 &         96.043 &          96.429 &          95.656 &        82.345 \\
  0.000 &   0.031 &   0.016 &         95.355 &          95.728 &          94.982 &        82.130 \\
  0.000 &   0.031 &   0.020 &         94.478 &          94.930 &          94.026 &        81.513 \\
  0.000 &   0.031 &   0.024 &         93.150 &          93.644 &          92.656 &        80.641 \\
  0.000 &   0.031 &   0.027 &         91.909 &          92.432 &          91.387 &        79.523 \\
  0.000 &   0.031 &   0.031 &         90.734 &          91.275 &          90.193 &        78.305 \\
  0.000 &   0.031 &   0.035 &         88.915 &          89.498 &          88.332 &        77.043 \\
\end{longtable}

\endgroup

\begingroup
\small
\setlength\tabcolsep{3pt}
\setlength\LTleft{\fill}\setlength\LTright{\fill}
\setlength{\LTcapwidth}{\linewidth}
\begin{longtable}{rrrrrrr}
\caption{ResNet50D results on MNIST.}
\label{tab:r50d_mnist_long}\\
\toprule
 Delta1 &  Delta2 &  Delta3 &  accuracy\_mean &  accuracy\_upper &  accuracy\_lower &  setsize\_mean \\
\midrule
\endfirsthead
\caption[]{ResNet50D results on MNIST.} \\
\toprule
 Delta1 &  Delta2 &  Delta3 &  accuracy\_mean &  accuracy\_upper &  accuracy\_lower &  setsize\_mean \\
\midrule
\endhead
\midrule
\multicolumn{7}{r}{{Continued on next page}} \\
\midrule
\endfoot

\bottomrule
\endlastfoot
  0.000 &   0.031 &   0.008 &         89.822 &          90.407 &          89.237 &         8.331 \\
  0.000 &   0.031 &   0.012 &         89.842 &          90.405 &          89.278 &         8.331 \\
  0.000 &   0.031 &   0.016 &         89.819 &          90.364 &          89.275 &         8.331 \\
  0.000 &   0.031 &   0.020 &         89.782 &          90.351 &          89.213 &         8.331 \\
  0.000 &   0.031 &   0.024 &         89.763 &          90.279 &          89.246 &         8.331 \\
  0.000 &   0.031 &   0.027 &         89.833 &          90.367 &          89.298 &         8.332 \\
  0.000 &   0.031 &   0.031 &         89.791 &          90.333 &          89.249 &         8.332 \\
  0.000 &   0.031 &   0.035 &         89.761 &          90.320 &          89.203 &         8.332 \\
\end{longtable}

\endgroup

\begingroup
\small
\setlength\tabcolsep{3pt}
\setlength\LTleft{\fill}\setlength\LTright{\fill}
\setlength{\LTcapwidth}{\linewidth}
\begin{longtable}{rrrrrrr}
\caption{ResNet50D results on TINYIMAGENET.}
\label{tab:r50d_tinyimagenet_long}\\
\toprule
 Delta1 &  Delta2 &  Delta3 &  accuracy\_mean &  accuracy\_upper &  accuracy\_lower &  setsize\_mean \\
\midrule
\endfirsthead
\caption[]{ResNet50D results on TINYIMAGENET.} \\
\toprule
 Delta1 &  Delta2 &  Delta3 &  accuracy\_mean &  accuracy\_upper &  accuracy\_lower &  setsize\_mean \\
\midrule
\endhead
\midrule
\multicolumn{7}{r}{{Continued on next page}} \\
\midrule
\endfoot

\bottomrule
\endlastfoot
  0.000 &   0.031 &   0.008 &         94.616 &          95.061 &          94.170 &       173.363 \\
  0.000 &   0.031 &   0.012 &         93.821 &          94.262 &          93.379 &       171.958 \\
  0.000 &   0.031 &   0.016 &         93.158 &          93.614 &          92.702 &       170.702 \\
  0.000 &   0.031 &   0.020 &         92.467 &          92.935 &          91.998 &       169.504 \\
  0.000 &   0.031 &   0.024 &         91.806 &          92.328 &          91.284 &       168.403 \\
  0.000 &   0.031 &   0.027 &         91.072 &          91.636 &          90.509 &       167.274 \\
  0.000 &   0.031 &   0.031 &         90.510 &          91.050 &          89.971 &       166.433 \\
  0.000 &   0.031 &   0.035 &         89.909 &          90.420 &          89.398 &       165.565 \\
\end{longtable}

\endgroup

\newpage

\subsection{Figures for Resnet50d}

\begin{figure}[!ht]
    \centering
    \includegraphics[width=0.7\columnwidth]{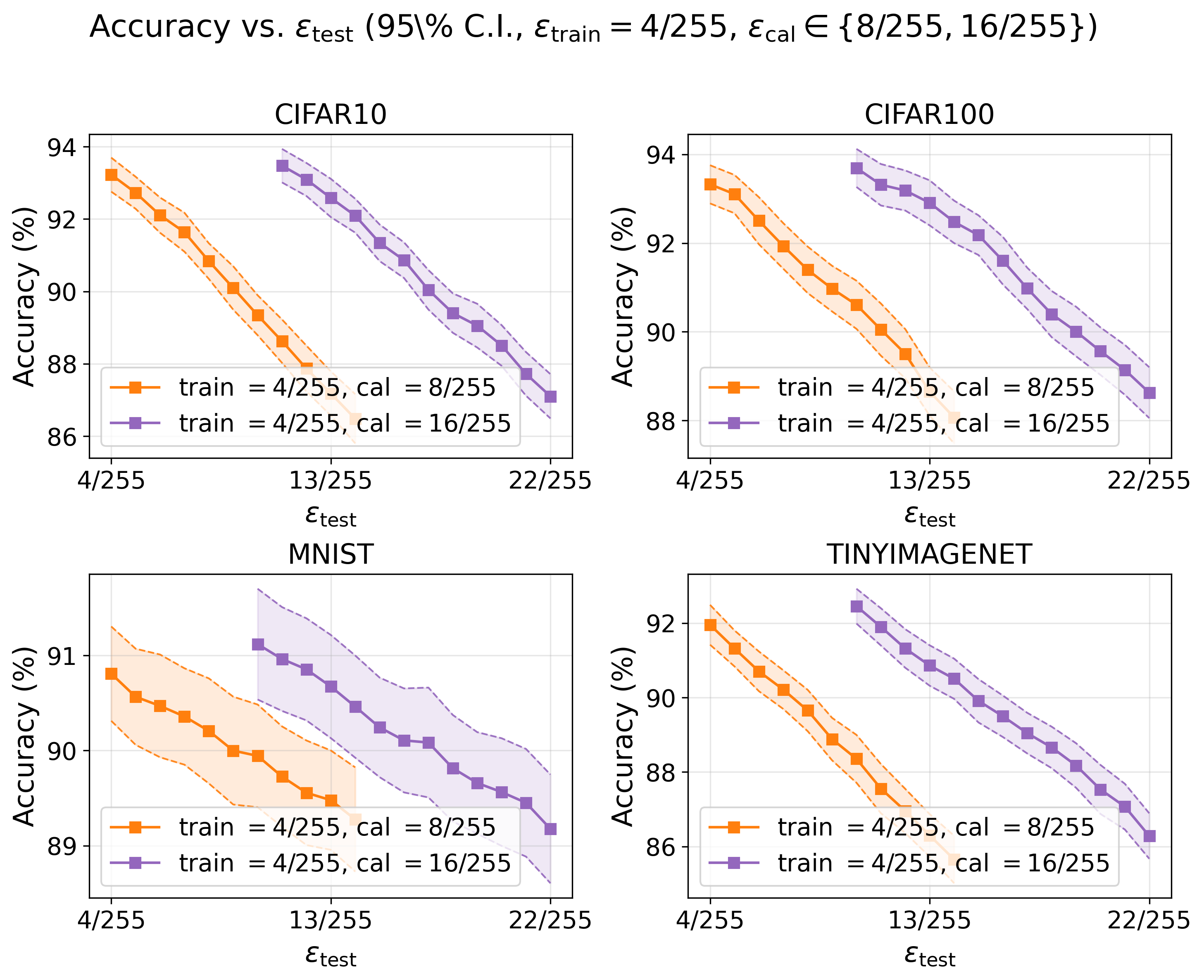}
    \caption{Resnet50d Accuracy}
    \label{fig:resnet_accuracy}
\end{figure}

\begin{figure}[!ht]
    \centering
    \includegraphics[width=0.7\columnwidth]{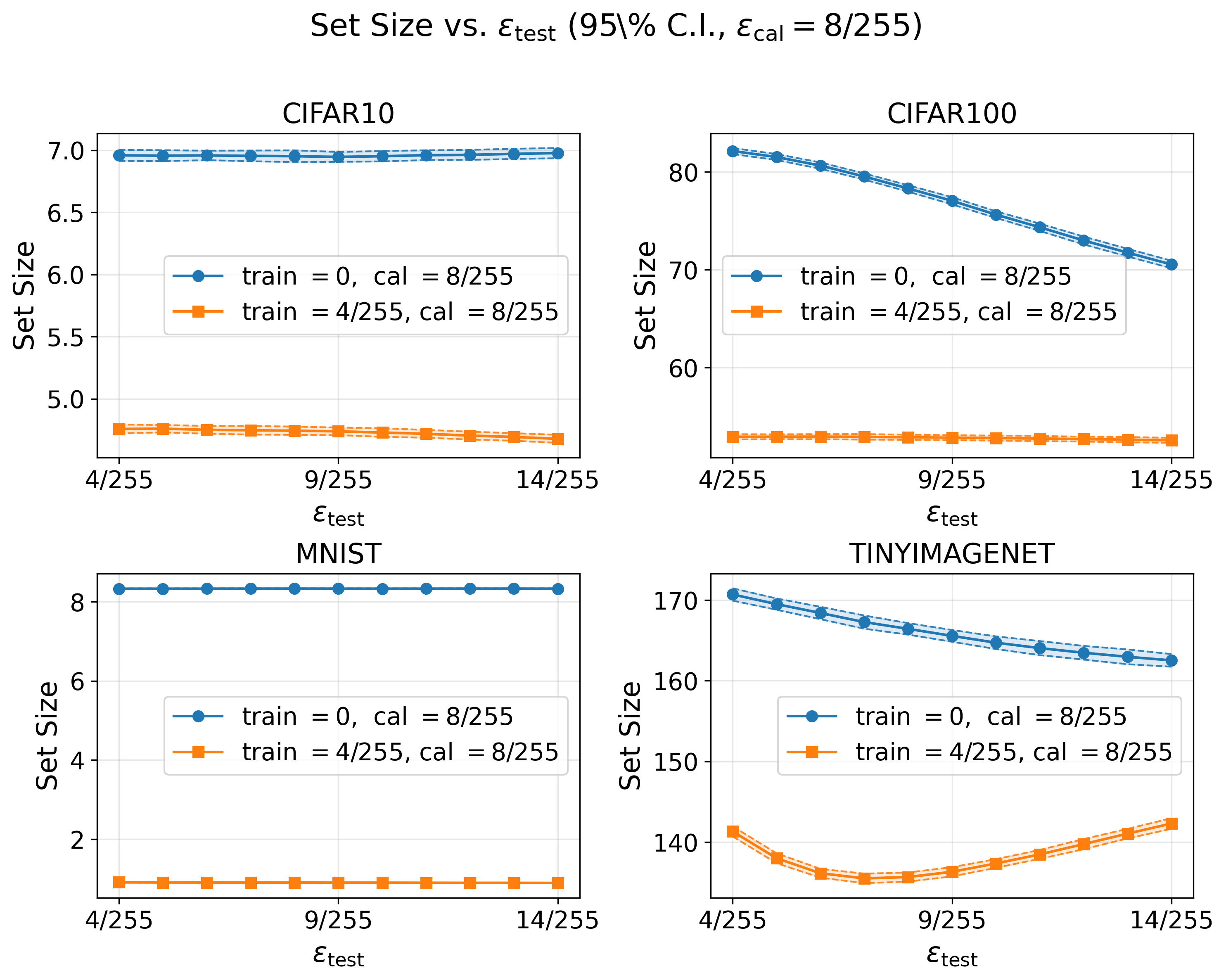}
    \caption{Resnet50d Set Size }
    \label{fig:resnet_setsize}
\end{figure}

\noindent\textbf{ResNet50d}
Across CIFAR10/100, MNIST, and TinyImageNet, the qualitative trends mirror our ViT results:
\begin{itemize}
  \item \textit{Accuracy vs.\ \(\epsilon_{\text{test}}\).} Accuracy decreases monotonically as \(\epsilon_{\text{test}}\) increases, while a stronger calibration set (\(\epsilon_{\text{cal}}=16/255\) vs.\ \(8/255\)) shifts the curves upward and enlarges the \(\epsilon_{\text{test}}\) range that meets the target accuracy band (95\% CIs shown).
  \item \textit{Effect of adversarial training on set size.} Under \(\epsilon_{\text{train}}=4/255\), adversarial training produces substantially smaller and more stable conformal prediction sets than clean training on CIFAR10/100 and MNIST, again echoing ViT.
\end{itemize}

\noindent\textit{Summary.} ResNet50d reproduces the ViT conclusions: stronger calibration improves test-time robustness, and adversarial training stabilizes and shrinks conformal set sizes while maintaining comparable accuracy.

\end{document}